\documentclass[twoside,11pt]{article}

\usepackage{jmlr2e}
\usepackage[cmex10]{amsmath}
\usepackage{amsmath,setspace,enumerate,amsbsy,amssymb} 
\DeclareMathAlphabet{\mathcal}{OMS}{cmsy}{m}{n}
\newcommand{\bs}[1]{\boldsymbol{#1}}
\newcommand{\VV}{\bs{\mathcal{V}}}
\newtheorem{property}{Property}

\firstpageno{1}

\begin{document}
\title{Gradient of Probability Density Functions based Contrasts for Blind Source Separation (BSS)}

\author{\name Dharmani Bhaveshkumar C. \hfill 				\email dharmanibc@gmail.com \\
       \addr Dhirubhai Ambani Institute of Technology (DAIICT), \\
       Gandhinagar, Gujarat, INDIA - 382001}

\editor{}

\maketitle

\begin{abstract}
The article derives some novel independence measures and contrast functions for Blind Source Separation (BSS) application. For the $k^{th}$ order differentiable multivariate functions with equal hyper-volumes (region bounded by hyper-surfaces) and with a constraint of bounded support for $k>1$, it proves that equality of any $k^{th}$ order derivatives implies equality of the functions.  The difference between  product of marginal  Probability Density Functions (PDFs) and  joint PDF of a random vector is defined as Function Difference (FD) of a random vector. Assuming the PDFs are $k^{th}$ order differentiable, the results on generalized functions are applied to the independence condition. This brings new sets of independence measures and BSS contrasts based on the $L^p$-Norm, $ p \geq 1$ of - FD,  gradient of FD (GFD) and Hessian of FD (HFD). Instead of a conventional two stage indirect estimation method for joint PDF based BSS contrast estimation, a single stage direct estimation of the contrasts is desired. The article targets both the efficient estimation of the proposed contrasts and extension of the  potential theory for an information field. The potential theory has a concept of reference potential and it is used to derive closed form expression for the relative analysis of potential field. Analogous to it, there are introduced concepts of Reference Information Potential (RIP) and Cross Reference  Information Potential (CRIP) based on the potential due to kernel functions placed at selected sample points as  basis in kernel methods.  The quantities are used to derive closed form expressions  for information field analysis using least squares. The expressions are derived through multiplicative kernel basis in two ways:    (a)  basis placed at the selected paired sample points (b) basis placed at the selected paired or un-paired sample points. The expressions are used to estimate $L^2$-Norm of FD and $L^2$-Norm of GFD based contrasts. 
Often, the performance of kernel based estimation methods is affected by the choice of a suitable bandwidth parameter. Usually, the choice of a bandwidth parameter is a compromise between accuracy of estimation and computation. The article uses data dependent Extended Rule-of-Thumb (ExROT) for bandwidth selection that balances both accuracy and computation. The higher order cumulants based ExROT helps achieving parameter free estimation method.   
Finally, the contrasts are verified for source separation by obtaining their optimization landscapes for two sources on varying distributions.   
\end{abstract}

\section{Introduction} 
\label{introcontrasts}
Contrast functions or simply contrasts\footnote{The formal definition is in Section \ref{contrasts}.} are the optimization functions to assure blind separation of unobserved sources from the available observation mixtures, when maximized.  The independence definition, its various   interpretations and their approximations are used to derive contrasts. 

The initial phase of research on BSS contrasts focused on the Shanon entropy and Kullback-Leibler divergence (KLD) based information theoretic independence interpretations  and their approximations through higher order statistics \citep{Cover,Papoulis91}. The other significant group of contrasts came from non-Gaussianity interpretations of independence and their approximations \citep{ICAbook01}. More details on these widely used, conventional contrast functions 
 can be found in \citep{Cardoso99NC,ContrastsICA,Hyva97NNSP}. 

The  research towards new contrasts for BSS has the following motivations. 
\begin{enumerate}	
\item More accurate BSS solution seems an everlasting hunger. So,  just out of  mathematical vigor to search for a  more accurate solutions, new contrasts are always of interest.
\item The Shanon entropy based contrasts are found to have spurious local optima 
\citep{Boscolo04,VrinsV05,PhaVri05}. Therefore, the contrast functions without the existence of spurious local optima are desired. 
\item The  large scale in BSS requires balancing accuracy with computation. This has motivated direct and fast estimation methods to derive contrasts \citep{FastHPham03,FastHPham04,leastsqICA11}. 
\item Some BSS contrasts with their estimation methods are biased towards a parametric family, say, subgaussian or supergaussian. To achieve unbiased estimation of sources, the focus has shifted to BSS using kernel based non-parametric estimation of various independence measures, 
as in Nonparametric ICA (NPICA)  \citep{Boscolo04} and kernel ICA (kICA) \citep{bach2003kernel}.  
\item The use of `prior' information with the independence assumption may find better estimations of the actual sources. Therefore, the contrast functions incorporating more generalized priors without violating the blind assumptions, other than the application specific priors used in Bayesian approach for BSS and semi-BSS problems,   are of interest. The bounded support assumption is one of such assumptions, used by many geometry based ICA and BSS algorithms \citep{GeomICA03,minRBSEVrins07}.
\end{enumerate}
Overall, the contrasts giving more accuracy at low computation, blind and without local minima are still in demand and open for further research. 

To overcome this demand, the latest trend in BSS contrasts follows two directions.
\begin{enumerate} 
\item Other than the conventional Shanon entropy and  KLD as a divergence measure between two PDFs, 
there exists many alternative definitions and interpretations of entropy, PDF distance measures and  independence interpretations \citep{ullah1996entropy,ITL2010,unifiedQIM11}. 
Inspired by the above motivations, the research community has started focusing on these alternatives 
to derive new BSS contrasts \citep{bach2003kernel,learned2003ica}. 
\item  The new independence interpretation  should be incorporated with kernel based fast and nonparametric  estimation technique to derive new BSS contrast. 
\end{enumerate}

Combining both the above directions, the latest trend is to use quadratic measures of independence for BSS. 
The article by \citet{achard2003quadratic} uses $L^2$ distance between the transformed characteristic functions of joint and product of the marginal PDFs. The Information Theoretic Learning (ITL) suggests many such quadratic independence measures, for example, Renyi's Entropy, Cross Information Potential (CIP), Euclidean distance ($D_{ED}$) based and Cauchy-Schwarz distance ($D_{CS}$) based  Quadrature Mutual Information (QMI) \citep{ITL2010,Renyi4BSS01}. 
The article by \citet{unifiedQIM11} provides ITL based unified framework to those quadratic distance measures  and proposes a new parameter free distance measure for ICA.

The current article is inspired by  all the above motivations and follows the latest trend. It derives some new independence interpretations relating gradient of the PDFs, specifically for bounded support random variables, and proposes new BSS contrasts. It achieves their nonparametric estimation with reduced computation by using least squares based  direct estimation approach.  The suitable choice of a kernel bandwidth parameter using data dependent bandwidth selection  \textit{Extended Rule-of-Thumb} by \citep{DharmaniExROTarxiv} achieves a parameter free contrast estimation. 
 
There have been proved some results for generalized differentiable multivariate functions. Looking PDFs as a generalized functions, the results are applied on independence of random vectors. 
The results are: 1) The equality of the gradient of joint probability density function (PDF) and the gradient of  product of the marginal PDFs imply independence. 2) The equality of the Hessian of joint PDF and the Hessian of  product of the marginal PDFs imply independence, if  the prior given that the random vector has bounded support i.e. its probability outside certain region is zero. 
These new independence interpretations 
are used to derive new independence measures and contrast functions for BSS. 
The bounded support condition is not very restricting. The reasons are: 1) Empirically, the sampled region is always bounded. 2) Numerically, the computers always work with definite range.  
Though may not be always, it might be a valid approach in most cases to blindly consider PDF outside the bounded sampled region to be zero. 
To achieve nonparametric estimation of the newly  derived contrasts, there has been derived single stage direct estimation method using least squares. To take the advantage of the quadratic nature of the contrasts, there are defined concepts of Reference Information Potential (RIP) and Cross RIP (CRIP) that depend upon IP due to selected kernel basis. The concepts are used to achieve closed form expressions for information field analysis. 
The derived closed form expression are verified by applying them to obtain $L^2$-Norm of FD and $L^2$-Norm of GFD contrasts. 
The method uses Gaussian kernels as basis and has two variations. One, the basis are placed at the selected paired sample points only. 
Another, the basis are placed at selected sample points may be paired or unpaired.


The next Section \ref{genfx} derives some results for generalized multivariate differentiable functions with bounded support. The results are applied to statistical independence condition in Section \ref{indm}. To better exploit the results, it derives new definitions and their important properties. Corresponding to that, 
the new independence measures are derived in Section \ref{indMsection}. The Section \ref{linearBSS} briefs the BSS problem  and the possible approach for solution.  The previous results are used to derive new BSS contrasts; satisfying the important properties of Scale invariance, Dominance and Discrimination; in  Section \ref{contrasts}.. There is also done local minima analysis of the derived contrast. 
The next Section \ref{estimate} discusses the  contrast function estimation approaches and derives prerequisites of Kernel Theory and Information Potential (IP). 
The Section \ref{ripcrip} defines the Reference IP (RIP) and related concepts. 
Then, the Section \ref{lsfd0} derives the least squares based closed form expression for information field analysis. The expressions are used to derive FD based estimators LSFD and LSFD2 in Section \ref{lsfd1} and GFD based estimators LSGFD and LSGFD2 in Section \ref{seclsgfd}. The Section \ref{verifyexp} reports empirical verification of the derived independence measures and BSS contrasts. 
The Section \ref{lsfdparameter} provides important discussion on required parameter selection for the derived estimators. Finally, the article ends with conclusion 
in Section \ref{newIMconclusion}.
\section{Some Results On the Equality of Generalized Constrained Multivariate Functions} 
\label{genfx}
\begin{definition}
A function $f: \mathbb{R}^n \rightarrow f(\mathbb{R}^n)$ is said to have support $\mathcal{R}$ if  $f(\mathbf{x}) = 0, \forall \mathbf{x}\in \mathcal{R}'$, where, $\mathcal{R}\subseteq \mathbb{R}^n$ and $\mathcal{R}'$ is its complement set. It is represented as  $supp(f)= \mathcal{R}$. 
Any superset of $\mathcal{R}$ is also a support.  
If $\mathcal{R}$ is bounded above and bounded below then $f$ is a said to be a bounded support function.
\end{definition}
Let  $\mbox{Conv}(\mathcal{R})$ be the convex hull of $\mathcal{R}$ that contains all convex combinations of points in $\mathcal{R}$.  Then, the definition says that for the bounded support functions both the support $\mathcal{R}$ and its  convex hull $\mbox{Conv}(\mathcal{R})$ have finite measures.  
If $\mathcal{R}$ is convex, both the support measure ($l(\mathcal{R})$) and its range ($l(\mbox{Conv}(\mathcal{R}))$) are same, where $l$ is the length of an interval. For example: let $\mathcal{R} = [-1,1]$. Then, the support measure $l(\mathcal{R})$  and the range $l(\mbox{Conv}(\mathcal{R}))=2$.  Now, let $\mathcal{R} = [-1,1]\bigcup (2,4]\setminus 3$. Then, $l(\mathcal{R})$ is 4. But, the $\mbox{Conv}(\mathcal{R})$ is $[-1,4]$ and $l(\mbox{Conv}(\mathcal{R}))=5$. 

For differentiable multivariate  functions with equal hyper volumes (region bounded by hyper surfaces) the following results are derived. For some of the results, an added constraint of random vector having bounded support  is required. 
\begin{theorem} 
\label{der1newth}
Let  $f: \mathbb{R}^n \rightarrow \mathbb{R}$, $g: \mathbb{R}^n \rightarrow \mathbb{R}$ and both satisfy the following conditions:
\begin{enumerate}
\item They are differentiable. 
\item $ \int_{\mathbb{R}^n} f(\mathbf{x})d\mathbf{x} =  \int_{\mathbb{R}^n} g(\mathbf{x})d\mathbf{x}, \mbox{ }\mathbf{x} \in \mathbb{R}^n $	
\end{enumerate}
Then, the following holds:
\begin{align}
\nabla f(\mathbf{x}) = \nabla g(\mathbf{x})  \Rightarrow f(\mathbf{x})=g(\mathbf{x})
\end{align}
\end{theorem}
\begin{proof}
Let us prove this Theorem by mathematical induction.\\
\textbf{The Base Case:} $n=1$ \\
Given $\int_{\mathbb{R}^n}{f(x)dx} = \int_{\mathbb{R}^n}{g(x)dx}$ and $\frac{d}{dx}f(x) = \frac{d}{dx}g(x)$.\\ 
Integrating both the sides of the latter equation leads to,
\begin{eqnarray}
  \label{eq1}
   f(x)  &=& g(x) + c 
\end{eqnarray} 
where, $c$ is some arbitrary constant. \\
Integrating both the sides of Equation \eqref{eq1} with respect to $x$ from $-\infty$ to $\infty$, brings: $$\int_{-\infty}^{\infty}c=0 \Rightarrow c = 0$$ \\
This proves the Theorem for the base case. \\
\textbf{The induction step: } \\
Given the Theorem holds for $n=k$, let us prove it for  $n=k+1$.\\
For the sake of simplicity and without the loss of generality, let us prove it for $n=3$ assuming it holds for $n=2$ i.e. for $k=2$. \\
Accordingly, let $\mathbf{x}= (x_1,x_2,x_3)^T$. \\
Given, $ \int_{\mathbb{R}^n} f(\mathbf{x})d\mathbf{x} =  \int_{\mathbb{R}^n} g(\mathbf{x})d\mathbf{x} $ 
and $\nabla f(\mathbf{x}) = \nabla g(\mathbf{x}) $.\\
As per the latter equation,  $ \frac{\partial}{\partial x_1}f(\mathbf{x}) = \frac{\partial}{\partial x_1}g(\mathbf{x})$. \\
Integrating both the sides with respect to 
$x_1$ leads to: 
\begin{align}
\label{neweq1}
f(\mathbf{x}) &= g(\mathbf{x}) + c(x_2, x_3) 
\end{align}
where, $c(x_2,x_3)$ is some arbitrary function of $x_2$ and $x_3$. \\
Taking partial derivative of Equation \eqref{neweq1} with respect to $x_2$, we get: $\frac{\partial}{\partial x_2}{c(x_2,x_3) = 0}$ \\
Taking partial derivative of Equation \eqref{neweq1} with respect to $x_3$, we get: $\frac{\partial}{\partial x_3}{c(x_2,x_3) = 0}$ \\
Combining the results on $c(x_2,x_3)$, we get: $c(x_2,x_3)= 0$ \\
This proves the Theorem  for $n=k+1$.\\
Combining both the base case and inductive step, by mathematical induction, the Theorem holds for all natural n.
\end{proof}
\begin{lemma} 
\label{der2newlm} 
Let  $f: \mathbb{R}^n \rightarrow \mathbb{R}$, $g: \mathbb{R}^n \rightarrow \mathbb{R}$ and both satisfy the following conditions:
\begin{enumerate}
\item They are second order differentiable. 
\item $ \int_{\mathcal{R}}^{} f(\mathbf{x})d\mathbf{x} =  \int_{\mathcal{R}}^{} g(\mathbf{x})d\mathbf{x}, \mbox{ }\mathbf{x} \in \mathbb{R}^n, \mbox{ }\mathcal{R} = \mbox{supp}(f)\bigcup \mbox{supp}(g)$
\item They have bounded support. 
\end{enumerate}
Then, the following holds:
\begin{align}
\nabla^2 f(\mathbf{x}) = \nabla^2 g(\mathbf{x}) \Rightarrow f(\mathbf{x})=g(\mathbf{x})
\end{align}
\end{lemma}
\begin{proof}
Let us prove this Lemma by mathematical induction.\\
\textbf{The Base Case:} $n=1$ \\
 Without loss of generality, let $\mathbf{I} = [-a,a] \supseteq \mbox{Conv}(\mathcal{R}), a \in \mathbb{R}$\\
Given $\int_{\mathbf{I}}{f(x)dx} = \int_{\mathbf{I}}{g(x)dx}$ and $\frac{d^2}{dx^2}f(x) = \frac{d^2}{dx^2}g(x)$.
Double integrating both the sides of latter equation with respect to $x$ leads to,
\begin{eqnarray}
  \label{eq1d2}
   f(x)  &=& g(x) + c_1x + c_2  
\end{eqnarray} 
where, $c_1$ and $c_2$ are some arbitrary constant.\\
Integrating both the sides of Equation \eqref{eq1d2} with respect to $x$ from $-a$ to $a$, brings $c_2=0$. \\
Integrating both the sides of Equation \eqref{eq1d2} with respect to $x$ from $-a$ to $b$, $b>a,, b \in \mathbb{R}$ brings $c_1=0$. \\
This proves the Lemma for the base case. \\
\textbf{The induction step: }\\
 Given the Lemma holds for $n=k$, let us prove it for  $n=k+1$.\\
For the sake of simplicity, let us prove it for $k=2$ i.e. $n=3$, assuming it holds for $n=2$. Its generalization to $k>2$ is obvious. \\
Without loss of generality, let $\mathbf{x}= (x_1,x_2,x_3)^T$ and $\mathbf{I} = [-a,a]^3 \supseteq \mbox{Conv}(\mathcal{R}), a \in \mathbb{R}$\\
Given $ \int_{\mathbf{I}}^{} f(\mathbf{x})d\mathbf{x} =  \int_{\mathbf{I}}^{} g(\mathbf{x})d\mathbf{x} $; 
 $\nabla^2 f(\mathbf{x}) = \nabla^2 g(\mathbf{x})$.\\
$\Rightarrow \frac{\partial^2}{\partial {x_1}^2}f(\mathbf{x}) = \frac{\partial^2}{\partial {x_1}^2}g(\mathbf{x})$.
Integrating twice both the sides with respect to $x_1$ leads to: 
\begin{align}
\label{d2neweq1}
f(\mathbf{x}) &= g(\mathbf{x}) + c_1(x_2, x_3)x_1 + c_2(x_2,x_3) 
\end{align}
 where, $c_1(x_2,x_3)$ and $c_2(x_2,x_3)$ are some arbitrary functions of $x_2$ and $x_3$. \\
Integrating Equation \eqref{d2neweq1} over $\mathbf{I}$, we get: $\int_{x_3}\int_{x_2}{c_2(x_2,x_3)}dx_2dx_3 = 0$\\
Integrating Equation \eqref{d2neweq1} with respect to $x_1$ from $-a$ to $a$, we get: 
\begin{align}
\label{d2neweq2}
f_1(x_2,x_3) & = g_1(x_2,x_3) + 2ac_2(x_2, x_3) 
\end{align}
where, $f_1(x_2,x_3)= \int_{-a}^{a}f(\mathbf{x})dx_1$ and $g_1(x_2,x_3)= \int_{-a}^{a}g(\mathbf{x})dx_1$.  \\
Integrating Equation \eqref{d2neweq2} with respect to both $x_2$ and $x_3$, we get:  $\int_{x_3}\int_{x_2}{f_1(x_2,x_3)dx_2dx_3} = \int_{x_3}\int_{x_2}{g_1(x_2,x_3)dx_2dx_3}$\\
Integrating $ \frac{\partial^2}{\partial {x_2}^2}f(\mathbf{x}) = \frac{\partial^2}{\partial {x_2}^2}g(\mathbf{x})$ with respect to $x_1$ from $-a$ to $a$, we get:  
$ \frac{\partial^2}{\partial {x_2}^2}f_1(x_2,x_3) = \frac{\partial^2}{\partial {x_2}^2} g_1(x_2,x_3)$ \\
Integrating $ \frac{\partial^2}{\partial {x_3}^2}f(\mathbf{x}) = \frac{\partial^2}{\partial {x_3}^2}g(\mathbf{x})$ with respect to $x_1$ from $-a$ to $a$, we get: $ \frac{\partial^2}{\partial {x_3}^2}f_1(x_2,x_3) = \frac{\partial^2}{\partial {x_3}^2} g_1(x_2,x_3)$\\
Applying, $n=2$ case, with all conditions satisfied, we get: $ f_1(x_2,x_3) =  g_1(x_2,x_3)$\\
Therefore, from Equation \eqref{d2neweq2}, $c_2(x_2,x_3)=0$. \\
Integrating the Equation \eqref{d2neweq1} with respect to $x_1$ from $-a$ to $b$, $b>a, b \in \mathbb{R}$, we get: $c_1(x_2,x_3)=0$ \\
This proves the Lemma for $n=k+1$. \\
Combining both the base case and inductive step, by mathematical induction, the Lemma for all natural n.
\end{proof}
\begin{lemma} 
\label{derpnewlm} 
Let  $f: \mathbb{R}^n \rightarrow \mathbb{R}$, $g: \mathbb{R}^n \rightarrow \mathbb{R}$ and both satisfy the following conditions:
\begin{enumerate}
\item They are $p^{th}$ order differentiable. 
\item $ \int_{\mathcal{R}}^{} f(\mathbf{x})d\mathbf{x} =  \int_{\mathcal{R}}^{} g(\mathbf{x})d\mathbf{x}, \mbox{ }\mathbf{x} \in \mathbb{R}^n, \mbox{ }\mathcal{R} = \mbox{supp}(f)\bigcup \mbox{supp}(g)$
\item They have bounded support. 
\end{enumerate}
Then, the following holds:
\begin{align}
\nabla^p f(\mathbf{x}) = \nabla^p g(\mathbf{x}) \Rightarrow f(\mathbf{x})=g(\mathbf{x})
\end{align}
\end{lemma}
\begin{proof}
The Theorem \ref{der1newth} proves this for $p = 1$ and the Lemma \ref{der2newlm} proves the same for $p = 2$.
Here, it needs be proved for any $p>2$.\\
Let us prove this Lemma by mathematical induction.\\
\textbf{The Base Case:} $n=1$ \\
 Without loss of generality, let $\mathbf{I} = [-a_1,a_1] \supseteq \mbox{Conv}(\mathcal{R}), a_1 \in \mathbb{R}$\\
Given $\int_{\mathbf{I}}{f(x)dx} = \int_{\mathbf{I}}{g(x)dx}$ and $\frac{d^p}{dx^p}f(x) = \frac{d^p}{dx^p}g(x)$.\\
Integrating $p$ times both the sides of latter equation with respect to $x$ leads to,
\begin{eqnarray}
  \label{eq1dp}
   f(x)  &=& g(x) + c_1x^{p-1} + c_2x^{p-2} + \dots + c_p  
\end{eqnarray} 
where, $c_i, i = \{1,2,\ldots,p \}$ are some arbitrary constant.\\
We can have two cases: Let $p$ be even.\\
Integrating both the sides of Equation \eqref{eq1dp} with respect to $x$ from $-a_1$ to $a_1$, brings $a_{11}c_2 + a_{12}c_4 + \ldots + a_{1q}c_p=0$, where $q= p/2$  and $a_{1i}$s are the coefficients as a result of integration.  \\
Let there be $q-1$ real numbers $a_i, i= \{2,3,\ldots q\}$  such that $a_i >  a_1$ and each one is different from the  other. Then, integrating \eqref{eq1dp} with respect to $x$ from $-a_i$ to $a_i$, brings over all $q$ equations with coefficients $a_{ij}, i=\{1,2,\ldots, q\}, j = \{1,2,\ldots,q\}$. Representing them in a matrix form, $\mathbf{Ac} = 0$, where $\mathbf{A}= [a_{ij}], \forall \mbox{ }a_{ij} \neq 0$ and $\mathbf{c} = (c_2, c_4, \ldots, c_p)^T$.  The  only solution to this equation is: $c_i = 0, i= \{2,4,\ldots,p\}$ i.e. all  $c_i, i = \forall \mbox{ even}$ in Equation \ref{eq1dp} are zero.	\\
Now, let there be $q$ real numbers $b_i > a_i, i = \{1,2,\ldots,q\}$ such that none of them is equal to the other. 
Integrating both the sides of Equation \eqref{eq1dp} with respect to $x$ from $-a_i$ to $b_i$ brings $b_{i1}c_1 + b_{i2}c_3 + \ldots + b_{iq}c_{p-1}=0$, where $q= p/2$ and  $b_{ij}, j = \{1,2,\ldots,q\}$ are the coefficients as a result of integration. In a matrix form,  $\mathbf{Bc}=0$, where $\mathbf{B}=[b_{ij}]$ and $\mathbf{c} = (c_1,c_3,\ldots, c_{p-1})^T$. This  brings all $c_i, i= \mbox{ odd}$ also to be zero.\\
This proves the lemma from Equation \eqref{eq1dp} for $p$ even case.\\
The $p$ odd case can also be solved similarly. \\
This proves the Lemma for the base case. \\
\textbf{The induction step: } \\
Given the Lemma holds for $n=k$, let us prove it for  $n=k+1$.\\
For the sake of simplicity, let us prove it for $k=2$ i.e. $n=3$, assuming it holds for $n=2$. \\
Its generalization to $k>2$ is obvious. \\
Without loss of generality, let $\mathbf{x}= (x_1,x_2,x_3)^T$ and $\mathbf{I} = [-a_1,a_1]^3 \supseteq \mbox{Conv}(\mathcal{R}), a \in \mathbb{R}$\\
Given $ \int_{\mathbf{I}}^{} f(\mathbf{x})d\mathbf{x} =  \int_{\mathbf{I}}^{} g(\mathbf{x})d\mathbf{x} $; 
 $\nabla^p f(\mathbf{x}) = \nabla^p g(\mathbf{x})$. \\
$\Rightarrow \frac{\partial^p}{\partial {x_1}^p}f(\mathbf{x}) = \frac{\partial^p}{\partial {x_1}^p}g(\mathbf{x})$.
Integrating $p$ times both the sides  with respect to $x_1$ leads to: 
\begin{align}
\label{dpneweq1}
f(\mathbf{x}) &= g(\mathbf{x}) + c_1(x_2, x_3)x_1^{p-1} + c_2(x_2,x_3)x_1^{p-2} + \ldots  + c_p(x_2,x_3)
\end{align}
 where, $c_i(x_2,x_3), i = \{1,2, \ldots,p\}$ are some arbitrary functions of $x_2$ and $x_3$. \\
Let $p$ be even. \\ 
Integrating Equation \eqref{dpneweq1} over $\mathbf{I}$, we get: \\
$\int_{x_3}\int_{x_2}{ \{ a_{11}c_2(x_2,x_3)x_1^{p-2} + a_{12}c_4(x_2,x_3)x_1^{p-4} + \ldots + a_{1q}c_p(x_2,x_3) \}dx_2dx_3} = 0$\\ where $q=p/2$ and $a_{1i}$ are the relevant coefficients.\\
Integrating Equation \eqref{dpneweq1} with respect to $x_1$ from $-a_1$ to $a_1$, we get: 
\begin{align}
\label{dpneweq2}
f_1(x_2,x_3) & = g_1(x_2,x_3) + a_{11}c_2(x_2,x_3)x_1^{p-2} + a_{12}c_4(x_2,x_3)x_1^{p-4} + \ldots + a_{1q}c_p(x_2,x_3) 
\end{align}
where, $f_1(x_2,x_3)= \int_{-a}^{a}f(\mathbf{x})dx_1$ and $g_1(x_2,x_3)= \int_{-a}^{a}g(\mathbf{x})dx_1$.  \\
Integrating Equation \eqref{dpneweq2} with respect to both $x_2$ and $x_3$, we get:  $\int_{x_3}\int_{x_2}{f_1(x_2,x_3)dx_2dx_3} = \int_{x_3}\int_{x_2}{g_1(x_2,x_3)dx_2dx_3}$\\
Integrating $ \frac{\partial^2}{\partial {x_2}^2}f(\mathbf{x}) = \frac{\partial^2}{\partial {x_2}^2}g(\mathbf{x})$ with respect to $x_1$ from $-a_1$ to $a_1$, we get:  
$ \frac{\partial^2}{\partial {x_2}^2}f_1(x_2,x_3) = \frac{\partial^2}{\partial {x_2}^2} g_1(x_2,x_3)$ \\
Integrating $ \frac{\partial^2}{\partial {x_3}^2}f(\mathbf{x}) = \frac{\partial^2}{\partial {x_3}^2}g(\mathbf{x})$ with respect to $x_1$ from $-a_1$ to $a_1$, we get: $ \frac{\partial^2}{\partial {x_3}^2}f_1(x_2,x_3) = \frac{\partial^2}{\partial {x_3}^2} g_1(x_2,x_3)$\\
Applying, $n=2$ case, with all conditions satisfied, we get: $ f_1(x_2,x_3) =  g_1(x_2,x_3)$\\
Therefore, from Equation \eqref{dpneweq2}, $a_{11}c_2(x_2,x_3)x_1^{p-2} + a_{12}c_4(x_2,x_3)x_1^{p-4} + \ldots + a_{1q}c_p(x_2,x_3)=0$. \\
Similar to the $n=1$ case, we can form $q-1$ such other independent equations, solving them we get:  $c_i=0. \forall i \mbox{ even}$	\\
Integrating the Equation \eqref{dpneweq1} with respect to $x_1$ from $-a_1$ to $b_1$, $b_1>a_1, b_1 \in \mathbb{R}$, we get:\\ $b_{11}c_1(x_2,x_3)x_1^{p-1} + b_{12}c_3(x_2,x_3)x_1^{p-3} + \ldots + a_{1q}c_{p-1}(x_2,x_3)x_1=0$ \\
Similar to the $n=1$ case, we can form $q-1$ such other independent equations, solving them we get:  $c_i=0. \forall i \mbox{ odd}$	\\
This proves the Lemma for $n=k+1$. \\
Combining both the base case and inductive step, by mathematical induction, the Lemma for all natural n.
\end{proof}
\begin{lemma} 
\label{derpnewth} 
Let  $f: \mathbb{R}^n \rightarrow \mathbb{R}$, $g: \mathbb{R}^n \rightarrow \mathbb{R}$ and both satisfy the following conditions:
\begin{enumerate}
\item They are $p^{th}$ order differentiable. 
\item $ \int_{\mathcal{R}}^{} f(\mathbf{x})d\mathbf{x} =  \int_{\mathcal{R}}^{} g(\mathbf{x})d\mathbf{x}, \mbox{ }\mathbf{x} \in \mathbb{R}^n, \mbox{ }\mathcal{R} = \mbox{supp}(f)\bigcup \mbox{supp}(g)$
\item They have bounded support. 
\end{enumerate}
Then, the following holds:
\begin{align}
f(\mathbf{x})=g(\mathbf{x})\Leftrightarrow \nabla^p f(\mathbf{x}) = \nabla^p g(\mathbf{x})
\end{align}
\end{lemma}
\begin{proof} Given  $f(\mathbf{x})$ and  $g(\mathbf{x})$ are differentiable: 
$ f(\mathbf{x})=g(\mathbf{x}) \Rightarrow \nabla^p f(\mathbf{x}) = \nabla^p g(\mathbf{x})$ \\
The converse part is proved in Lemma \ref{derpnewlm}. This proves the current Lemma.
\end{proof}
\begin{theorem} 
\label{derp2th} 
Let  $f: \mathbb{R}^n \rightarrow \mathbb{R}$, $g: \mathbb{R}^n \rightarrow \mathbb{R}$ and both satisfy the following conditions:
\begin{enumerate}
\item They are $p^{th}$ order differentiable. 
\item $ \int_{\mathcal{R}}^{} f(\mathbf{x})d\mathbf{x} =  \int_{\mathcal{R}}^{} g(\mathbf{x})d\mathbf{x}, \mbox{ }\mathbf{x} \in \mathbb{R}^n, \mbox{ }\mathcal{R} = \mbox{supp}(f)\bigcup \mbox{supp}(g)$
\item They have bounded support. 
\end{enumerate}
Then, the following holds:
\begin{align}
f(\mathbf{x})=g(\mathbf{x})\Leftrightarrow \nabla f(\mathbf{x}) = \nabla g(\mathbf{x}) \Leftrightarrow \ldots  \Leftrightarrow \nabla^p f(\mathbf{x}) = \nabla^p g(\mathbf{x}) 
\end{align}
\end{theorem}
\begin{proof} Applying principle of transitivity of implication (Hypothetical syllogism) to Lemma \ref{derpnewth}  with  varying values of $p$, this Theorem is proved.
\end{proof}
For a generalized functions, given any $p^{th}$ order derivatives are equal, the only available information would be that the functions differ by a constant in their $(p-1)^{th}$ order derivative. It would require $p$ initial conditions to decide about equality of the functions.  The Theorem \ref{der1newth} proves that if the given condition for $p=1$ is added with one more condition of equal hypervolumes then it brings equality of the functions. 
The above Theorem \ref{derp2th} proves further the strength of an added prior information that the function is also having bounded support. This prior implies that any $p^{th}$ order derivatives are equal, the functions are equal. 
Conversely, given two functions with equal $p^{th}$ derivative are not equal imply either of the conditions are not matching. For example; let $f(x)$ and $g(x)$  are constant functions with unequal constant values and unequal supports on real line such that area under them are same. The derivatives are same and zero everywhere. 
The example seems counterexample of the Theorem \ref{der1newth} as both derivatives are same but not the functions.
More better observation clears that both the functions are discontinuous at boundary points. This violates the  differentiability condition of Theorem. The derivative values given zero, actually excludes points with Lebesgue measure zero.  
\section{Applications of the Results On Independence}
\label{indm}
By definition, the probability density functions have area under the curve to be unity. 
The bounded support function assumption seems restricting application to many PDFs. But, as said in the Section \ref{introcontrasts}, empirically and numerically this assumption is not restricting. So, it is natural to think of extending the previous results to independence condition. 
Looking similarity with the results on Score Function Difference (SFD) and its properties related to independence in \citep{malmiri02Thesis}, the topic is developed using matching terminology.  \\
Let  $\mathbf{x} = (x_1, x_2,\ldots , x_n)^T$ is an n-dimensional random vector, where, $x_i, i=1,2,\ldots,n$ are random variables;   
$p_{\mathbf{x}}(x_1,x_2,\ldots,x_n)$ is the joint PDF of $\mathbf{x}$ and $\prod_{i=1}^{n} p_{x_i}(x_i)$ is the product of the marginal PDFs. For this description, the statistical independence as in \citep{Papoulis91}, and other terms are  defined.  
\begin{definition}[Statistical Independence] 
\label{inddef}
 The random variables $x_1, x_2,\ldots , x_n$ are said to be statistically independent, if 
\begin{align*}
p_{\mathbf{x}}(x_1,x_2,\ldots,x_n) = \prod_{i=1}^{n} p_{x_i}(x_i)
\end{align*}
\end{definition}
As the stastical independence finds many applications, it is worth defining the following term.
\begin{definition}[Function Difference (FD)] 
The Function Difference (FD) of $\mathbf{x}$ is the difference between product of its marginal PDFs $\prod_{i=1}^{n} p_{x_i}(x_i)$ and  its joint PDF   $p_{\mathbf{x}}(x_1,x_2,\ldots,x_n)$, that is:
\begin{align*}
\bs{\Delta}(\mathbf{x}) \stackrel{def}{=} \prod_{i=1}^{n} p_{x_i}(x_i) - p_{\mathbf{x}}(x_1,x_2,\ldots,x_n)
\end{align*}
\end{definition}
From the Definition,  $\bs{\Delta}(\mathbf{x}) \equiv 0$ implies independence.  
 
With the assumption that the joint PDF and the marginal PDFs are differentiable, the followings are defined. 
\begin{definition}[GPF]
The Gradient of the Product Function (GPF) of $\mathbf{x}$ is the gradient of the product of the marginal PDFs $\prod_{i=1}^{n} p_{x_i}(x_i)$, that is: 
\begin{align*}
\bs{\xi}(\mathbf{x}) &\stackrel{def}{=} \nabla \left(\prod_{i=1}^{n} p_{x_i}(x_i)\right) = \left(\xi_1(x_1), \xi_2(x_2), \ldots, \xi_n(x_n) \right)^T \\
& \mbox{where, } 
\xi_l(x_l) \stackrel{def}{=} \frac{\partial}{\partial x_l}\left(\prod_{i=1}^{n} p_{x_i}(x_i)\right)
\end{align*}
\end{definition}
\begin{definition}[GJF]
The Gradient of the Joint Function (GJF) of $\mathbf{x}$ is the gradient of the joint PDF $p_{\mathbf{x}}(x_1,x_2,\ldots,x_n)$, that is:
\begin{align*}
\bs{\zeta}(\mathbf{x}) & \stackrel{def}{=} \nabla p_{\mathbf{x}}(x_1,x_2,\ldots,x_n) = \left(\zeta_1(\mathbf{x}),\zeta_2(\mathbf{x}),\ldots,\zeta_n(\mathbf{x})\right)^T \\
&\mbox{where, } \zeta_l(\mathbf{x}) \stackrel{def}{=} \frac{\partial}{\partial x_l}p_{\mathbf{x}}(x_1,x_2,\ldots,x_n)
\end{align*} 
\end{definition}
\begin{definition}[GFD] The Gradient Function Difference (GFD) of $\mathbf{x}$ is the difference between its GPF and GJF or equivalently it is the gradient of FD, that is:
\begin{equation*}
\bs{\alpha}(\mathbf{x}) \stackrel{def}{=} \bs{\xi}(\mathbf{x}) - \bs{\zeta}(\mathbf{x}) = \nabla(\bs{\Delta}(\mathbf{x}))
\end{equation*}
\end{definition}
The following property proves that GFD ($\bs{\alpha}(\cdot)$) contains important information about  independence of the components of a random vector.
\begin{property} 
\label{property1} The components of a random vector $\mathbf{x} = (x_1, x_2,\ldots , x_n)^T$ are independent if and only if $\bs{\alpha}(\mathbf{x})\equiv 0$, that is:
\begin{equation}
\label{property1eq}
\bs{\xi}(\mathbf{x}) = \bs{\zeta}(\mathbf{x})
\end{equation} 
\end{property}
\begin{proof}
Let us take for better analogy, $f(\mathbf{x}) = p_{\mathbf{x}}(x_1,x_2,\ldots,x_N)$ and $g(\mathbf{x}) = \prod_{i=1}^{N} p_{x_i}(x_i)$ \\
$\Rightarrow \int_{ \mathcal{R}}^{}  f(\mathbf{x})d\mathbf{x} =  \int_{\mathcal{R}}^{} g(\mathbf{x})d\mathbf{x} =1$	\\
Given $\bs{\xi}(\mathbf{x}) = \bs{\zeta}(\mathbf{x})$ or $\nabla f(\mathbf{x})= \nabla g(\mathbf{x})$. \\
The required conditions are satisfied. So, applying Theorem \ref{der1newth} and the Definition \ref{inddef} of Independence, the property is proved.
\end{proof}
For the same random vector $\mathbf{x}$ with added assumptions that the joint PDF and the product of the marginal PDFs are both second order differentiable and have bounded support, the following definitions and results are obtained.
\begin{definition}[HPF]
The Hessian of the Product Function (HPF) of $\mathbf{x}$ is the Hessian of the product of the marginal PDFs $\prod_{i=1}^{n} p_{x_i}(x_i)$, that is: 
\begin{equation*}
\bs{\Xi}(\mathbf{x}) \stackrel{def}{=} \nabla \bs{\xi}(\mathbf{x}) =  \nabla^2 \left(\prod_{i=1}^{n} p_{x_i}(x_i)\right) 
\end{equation*}
\end{definition}
\begin{definition}[HJF]
The Hessian of the Joint Function (HJF) of $\mathbf{x}$ is the Hessian of the joint PDF $p_{\mathbf{x}}(x_1,x_2,\ldots,x_n)$, that is:
\begin{equation*}
\bs{Z}(\mathbf{x}) \stackrel{def}{=} \nabla \bs{\zeta}(\mathbf{x}) = \nabla^2 p_{\mathbf{x}}(x_1,x_2,\ldots,x_n)
\end{equation*} 
\end{definition}
\begin{definition}[HFD] The Hessian Function Difference (HFD) of $\mathbf{x}$ is the difference between its HPF and HJF, or equivalently it is the Hessian of FD, that is:
\begin{align*}
\bs{A}(\mathbf{x}) & \stackrel{def}{=} \bs{\Xi}(\mathbf{x})- \bs{Z}(\mathbf{x}) = \nabla \left(\bs{\xi}(\mathbf{x})- \bs{\zeta}(\mathbf{x})\right)  \\
& = \nabla^2(\bs{\Delta}(\mathbf{x})) = \nabla \bs{\alpha}(\mathbf{x}) 
\end{align*}
\end{definition}
The following property proves that HFD ($\bs{A}(\cdot)$) contains important information about independence of the components of a random vector.
\begin{property} 
\label{property2} The components of a bounded support random vector $\mathbf{x} = (x_1, x_2,\ldots , x_n)^T$ are independent if and only if $\bs{A}(\mathbf{x})\equiv 0$, that is:
\begin{equation}
\label{property2eq}
\bs{\Xi}(\mathbf{x}) = \bs{Z}(\mathbf{x})
\end{equation} 
\end{property}
\begin{proof}
Applying Lemma \ref{derpnewlm} with $p=2$, the property is proved. 
\end{proof}
\begin{corollary} 
\label{derp2cor} 
Let $\mathbf{x}= (x_1,x_2,\ldots,x_n)^T$ be an n-dimensional random vector; $p_{\mathbf{x}}(x_1,x_2,\ldots,x_n)$ be its joint PDF;  $\prod_{i=1}^{n} p_{x_i}(x_i)$ be its product of the marginal PDF; the PDFs be second order differentiable with bounded support  $\mathcal{R}\subseteq \mathbb{R}^n$. 
Then:
\begin{align}
p_{\mathbf{x}}(x_1,x_2,\ldots,x_n)=\prod_{i=1}^{n} p_{x_i}(x_i) \Leftrightarrow \bs{\xi}(\mathbf{x}) = \bs{\zeta}(\bs{x}) \Leftrightarrow \bs{\Xi}(\mathbf{x}) = \bs{Z}(\mathbf{x})
\end{align}
\end{corollary}
\begin{proof}
Applying Theorem \ref{derp2th} and the Definition \ref{inddef} of independence, 
the corollary is proved. 
\end{proof}
The Property \ref{property1} of GFD, Property \ref{property2} of HFD and the Corollary \ref{derp2cor} bring further interpretations on independence of bounded support random vector. Our goal is to develop new contrasts based on them. For that the quantities should be nonnegative to be quantified as measures. So, first let there be derived independence measures based on these results.
\section{Deriving new Independence Measures}
\label{indMsection} 
The goal here is to derive independence measures based on the quantities FD, GFD and HFD. 
But, the quantities do not assure nonnegativity to be quantified as measures. 
Assuming a class of $L^p$ integrable PDFs, the $L^p$ norm can be applied on them. Being norm, they satisfy all the  properties of a \textit{metric} and an added property of absolute scale invariance, as per the definition of norm. The details on the definitions of a measure, a \textit{metric}, a norm and the specific $L^p$-norm are briefed in Appendix \ref{aplpnorm}. 

\hspace{0.2 in}	It is desired that a distance measure between PDFs is invariant with respect to translation and scaling i.e. the deviation in mean and the variance should not affect the distance measure. 
The reason is, the nearness of the PDFs should imply their shapes are matching. 
The desired property of scale invariance, instead of the absolute scale invariance,  can be assured by defining an independence measure that applies a norm on normalized PDFs i.e. converting them first into zero mean, univariance PDFs. 
\begin{proposition}
For a random vector $\mathbf{x} \in \mathbb{R}^n$ with $L^p$ integrable joint and marginal PDFs,  
LpFD($\mathbf{x}$) or $\bs{\Delta}_{p}(\mathbf{x})$  defined as under 
is an independence measure.   
\begin{align}
\label{lpfdmeasure}
\bs{\Delta}_{p}(\mathbf{x}) & \stackrel{def}{=} ||\bs{\Delta}(\mathbf{z})||_{p} = \left(\int_{\mathbb{R}^n} \left|   \bs{\Delta}(\mathbf{z}) \right|^p d\mathbf{z}\right)^{\frac{1}{p}} \\
\mbox{ or }   d_p\left( \prod_{i=1}^{n}p_{x_i}(x_i), p_{\mathbf{x}}(\mathbf{x}) \right) 
& = \left(\int_{\mathbf{x}} \left| \prod_{i=1}^{n}p_{x_i}\left( \frac{x_i - \bar{x_i}}{\sigma_{x_i}} \right) - p_{\mathbf{x}}\left( \frac{\mathbf{x}-\bar{\mathbf{x}}}{\sigma_\mathbf{x}} \right) \right|^p d\mathbf{x}\right)^{\frac{1}{p}} 
\end{align}
where, $\mathbf{z} = \left( \frac{\mathbf{x}-\bar{\mathbf{x}}}{\sigma_\mathbf{x}} \right)$,  $\bar{\mathbf{x}}$ and $\bar{x_i}$ are consecutively  mean of $\mathbf{x}$ and $x_i$, $\sigma_{\mathbf{x}}$ and $\sigma_{x_i}$ are corresponding standard deviations. 
\end{proposition}
\begin{proof}
By definition, $\bs{\Delta}_{p}(\mathbf{x}) \geq 0$ and 
$\bs{\Delta}_{p}(\mathbf{x}) =  0$ if and only if $\bs{\Delta}(\mathbf{x}) \equiv 0$.
Also, by Definition \ref{inddef} of independence, $\bs{\Delta}_{p}(\mathbf{x}) = 0$ if and only if the components of $\mathbf{x}$ are independent.  \\
This proves that $\bs{\Delta}_{p}(\mathbf{x})$ is an independence measure. More specifically, it is an independence metric  with respect to $\bs{\Delta}(\mathbf{x})$, but not necessarily on the space of random vectors $\mathbf{x}$ themselves.
\end{proof}
\hspace{0.2 in}		
The GFD is essentially a vector, whose value is an n-tuple of functions. Accordingly,  $ \bs{\alpha}: L^p \times L^p \times \ldots L^p \rightarrow \mathbb{R}$. So,  $L^p$-norm can still be applied as under. 
\begin{proposition}
\label{alpha1p}
For an n-dimensional random vector $\mathbf{x} =  (x_1,x_2,...,x_n)$ with differentiable joint and marginal PDFs, LpGFD($\mathbf{x}$) or $\bs{\alpha}_{p}(\mathbf{x})$  defined as under 
is an independence measure. 
\begin{align}
\label{alpha1pdef}
\bs{\alpha}_{p}(\mathbf{x}) &\stackrel{def}{=} ||\bs{\alpha}(\mathbf{z})||_{p} = \left( \sum_{i=1}^{n}
\left( ||\alpha_i(\mathbf{z})||_{p} \right)^p \right)^\frac{1}{p} \\
\mbox{ or } d_p\left(\mathbf{\xi}(\mathbf{x}),\mathbf{\zeta}(\mathbf{x})\right)  &= \left( \sum_{i=1}^{n} \int_{\mathbf{z}}  \left| \xi_i\left( \frac{\mathbf{x}-\bar{\mathbf{x}}}{\sigma_\mathbf{x}} \right) - \zeta_i\left( \frac{\mathbf{x}-\bar{\mathbf{x}}}{\sigma_{\mathbf{x}}}\right) \right|^p d\mathbf{z} \right)^{\frac{1}{p}}
\end{align}
where, $\mathbf{z} = \left( \frac{\mathbf{x}-\bar{\mathbf{x}}}{\sigma_\mathbf{x}} \right)$, $\bar{\mathbf{x}}$ is the  mean of $\mathbf{x}$ and  $\sigma_{\mathbf{x}}$ is the corresponding standard deviation. 
\end{proposition}
\begin{proof}
The differentiable PDF condition, assures $L^p$ integrability.\\
By definition, $\bs{\alpha}_{p}(\mathbf{x}) \geq 0$ and $\bs{\alpha}_{p}(\mathbf{x}) =  0$ if and only if $\bs{\alpha}(\mathbf{x}) \equiv 0$.
Applying Property \ref{property1}, $\bs{\alpha}_{p}(\mathbf{x}) = 0$ if and only if the components of $\mathbf{x}$ are independent.  \\
This proves that $\bs{\alpha}_{p}(\mathbf{x})$ is an independence measure. More specifically, it is an independence metric  with respect to $\bs{\alpha}(\mathbf{x})$, but not necessarily on the space of random vectors $\mathbf{x}$ themselves.
\end{proof}
\hspace{0.2 in}		
The HFD is essentially a matrix. So, matrix norms are applicable. The `\textit{Entrywise}' norms treat matrix entries as a vector entries. The following independence measure can be defined. 
\begin{proposition}
\label{alpha2p}
For a bounded support random vector $\mathbf{x} =  (x_1,x_2,...,x_N)$ with second order differentiable joint and marginal PDFs, LpHFD($\mathbf{x}$) or $\bs{A}_{p}$ 
is an independence measure, where:
\begin{align}
\bs{A}_{p} & \stackrel{def}{=} ||\bs{A}(\mathbf{x})||_{p} = \left( \sum_{j=1}^{n}\sum_{i=1}^{n}
\left( ||A_{ij}(\mathbf{x})||_{p} \right)^p \right)^\frac{1}{p} \\
\mbox{or }  d_p\left(\bs{\Xi}(\mathbf{x}),\bs{Z}(\mathbf{x})\right) & = \left( \sum_{j=1}^{n}\sum_{i=1}^{n} \int_{x_{ij}}  \left| \Xi_{ij}\left( \frac{\mathbf{x}-\bar{\mathbf{x}}}{\sigma_\mathbf{x}}\right) - Z_{ij}\left(\frac{\mathbf{x}-\bar{\mathbf{x}}}{\sigma_\mathbf{x}}\right) \right|^p dx_{ij} \right)^{\frac{1}{p}} 
\end{align}
\end{proposition}
where, $\bar{\mathbf{x}}$ is the  mean of $\mathbf{x}$ and $\sigma_{\mathbf{x}}$ is the corresponding standard deviation. 
\begin{proof}
The second order differentiable PDF condition, assures $L^p$ integrability.\\
By definition, $\bs{A}_{p} \geq 0$ and $\bs{A}_{p}(\mathbf{x}) =  0$ if and only if $\bs{A}(\mathbf{x}) \equiv 0$. 
 Applying property \ref{property2}, $\bs{A}_{p} = 0$ if and only if the components of $\mathbf{x}$ are independent.  \\
This proves that $\bs{A}_{p}$ is an independence measure. More specifically, it is an independence metric  with respect to $\bs{A(\mathbf{x})}$, but not necessarily on the space of random vectors $\mathbf{x}$ themselves.
\end{proof}
\section{The Linear BSS Problem and Solution}
\label{linearBSS} 	
The Blind Source Separation (BSS) model explains generation of an observed random vector $\mathbf{x}(t)$, as an  transformation to another latent (hidden) random vector $\mathbf{s}(t)$. Assuming linear and instantaneous  transformation, mathematically, $\mathbf{x}(t) = \mathbf{A}\mathbf{s}(t)$,
where $\mathbf{x}(t) = [x_1(t); x_2(t); \ldots; x_m(t)]$; $\mathbf{s}(t) = [s_1(t); s_2(t); \ldots; s_n(t)]$; $x_i(t)$, $s_i(t)$ are random variables with values in $\mathcal{R}$; $m = n >= 2$ and  $\mathbf{A}$ is full rank. 
Let there be available N umber of samples of each observed random variable. Assuming an identicle distribution, the instantaneous model can be extended for N realizations. Let $\mathbf{X}(t) = [\mathbf{x}_1(t); \mathbf{x}_2(t); \ldots; \mathbf{x}_m(t)]$ be the $m \times N$ data or observation  matrix and $\mathbf{S}(t)$ be the $n \times N$ component or source  matrix. Then,
\begin{equation}
\mathbf{X}(t) = \mathbf{A}\mathbf{S}(t)
\end{equation} 
	The problem of BSS is to estimate both the unknowns $\mathbf{A}$ and $\mathbf{S}(t)$, with the only assumption of $\mathbf{s}_i(t)$ being mutually the \textit{most independent possible (m.i.p.)} random variables with respect to a given contrast. If $\mathbf{W}$ is the estimated inverse of the mixing matrix $\mathbf{A}$ then the estimated source or component matrix $\mathbf{Y}(t)$ is:
\begin{equation}
 \mathbf{Y}(t) = \mathbf{A}^{-1}\mathbf{X}(t) = \mathbf{W}\mathbf{X}(t) = \mathbf{W}\mathbf{A}\mathbf{S}(t)
 \end{equation}
As, $\mathbf{X}(t) = \mathbf{A}\mathbf{S}(t) = (\mathbf{A }\Lambda^{-1} \mathbf{P^{-1}})(\mathbf{P} \Lambda \mathbf{S}(t)) $, for any permutation matrix $\mathbf{P}$ and a scaling matrix $\Lambda$, there are going to be scaling and permutation ambiguities in the estimated components. 
 
\hspace{0.2 in} Given the unknown sources are \textit{independent and identically distributed} (\textit{i.i.d.}) with maximum one of them being Gaussian, 
a unique BSS  solution is assured by Darmois-Skitovtch Theorem \citep{Comon2010,Comon94,eriksson2004identifiability}.  Accordingly,  the BSS solution for linear, instantaneous mixing system can be obtained by maximizing the independence among $y_i(t)$s  with respect to the separation matrix $\mathbf{W}$, as:
\begin{equation}
\mathbf{y}^{*}(t) = \underset{\mathbf{W}}{\mbox{argmax}} \mbox{  } \Phi(\mathbf{y}(t)) 
\end{equation}
where, $\Phi(\mathbf{y}(t))$ is the optimization function, based on independence or dependence measure, that assures source separation on maximization. It is identified as a contrast function or simply a `contrast'. 
Oerall, the BSS solution demands a suitable  contrast function as an optimization criteria and a suitable optimization technique corresponding to that contrast function. 
\section{Deriving New Contrasts for ICA and BSS}
\label{contrasts}
A formal definition of contrasts, based on references \citep{Comon96} and \citep[Chapter 3]{Comon2010}, for BSS is as under. 
\begin{definition}[Contrast for BSS]
Let $\mathcal{H}$ be a set of static transformations (filters) containing an identity transformation (filter) $\mathbf{I}$; $\mathcal{S}$ be a set of source random variables that are independent and ; $\mathcal{X} = \mathcal{H} \cdot \mathcal{S}$ be the set of random variables obtained by the action of $\mathcal{H}$ on $\mathcal{S}$; $\Phi$ be a mapping from $\mathcal{H}\times \mathcal{H}\cdot \mathcal{S}$ to $\mathbb{R}$. Also, denoted by  $\mathcal{T}$ the set of trivial filters of $\mathcal{H}$, which leave criterion $\Phi$ unchanged. A mapping $\Phi(\mathbf{H};\mathbf{x})$ is a contrast if it depends solely on the PDF of $\mathbf{x}$ and if it satisfies the following three properties below.
\begin{align*}
&\mbox{\textbf{a.} Invariance: }
 \forall \mathbf{x} \in \mathcal{X}, \forall \mathbf{T} \in \mathcal{T}, \Phi(\mathbf{T};\mathbf{x}) = \Phi(\mathbf{I};\mathbf{x})\\
&\mbox{\textbf{b.} Dominance: } 
 \forall \mathbf{s} \in \mathcal{S}, \forall \mathbf{H} \in \mathcal{H}, \Phi(\mathbf{H};\mathbf{s}) \leq \Phi(\mathbf{I};\mathbf{s}) \\
&\mbox{\textbf{c.} Discrimination: } 
 \forall \mathbf{s} \in \mathcal{S}, \mbox{ } if \mbox{ } \mathbf{H} \in \mathcal{H} \mbox{ } satisfies \nonumber \\
&\mbox{   }  \Phi(\mathbf{H};\mathbf{s}) = \Phi(\mathbf{I};\mathbf{s}), \mbox{ }then \mbox{ } \mathbf{H} \in \mathcal{T} 
\end{align*}
\end{definition}
The Dominance  property assures that the actual sources have the global maxima.
  The Discrimination property assures that there is no other spurious solution achieving the global maxima. 
There is some discussion needed on the invariance property. It is obvious that the independence components found using a given measure, are still independent if permuted or scaled. So, one of the solutions is available,  whole class  of solutions related through permutation and scaling operation is available. The Invariance property assures this by stating that whole class should have a same measure. The widely used \textit{KL-divergence} assures this property. But, it is known that many other PDF divergence measures such as; Itakura-Saito distance, density-power divergences 
do not assure this scale invariance property. To accommodate such a larger class of divergences, without deteriorating the BSS performance, there has been first defined and then proposed relative  scale invariance property as the sufficient property with other properties to be quantified as contrast. 
\begin{definition}
The contrast $\Phi:\mathcal{H}\times \mathcal{H}\cdot \mathcal{S} \rightarrow \mathbb{R}$ is said to have relative Scale Invariance property; if it satisfies the following condition: Given $\mathbf{y} = \mathbf{\Lambda x}$
\begin{align*}
\label{relativescale}
\Phi(\mathbf{y}) = k(\Lambda)\Phi(\mathbf{x})
\end{align*}
where, $k(\Lambda)$ is a fixed transformation as a function of $\Lambda$.
\end{definition}
\begin{proposition}
 $\Phi:\mathcal{H}\times \mathcal{H}\cdot \mathcal{S} \rightarrow \mathbb{R}$  is a contrast for linear BSS, if it satisfies the Relative Scale Invariance property with other required properties satisfied.
\end{proposition}
\begin{proof}
The following arguments justify the proposition. 
\begin{itemize}
\item Given $\mathbf{T}\in \mathcal{T}$ is a scale matrix with diagonal entries only.  
As the source components are independent, $\Phi(\mathbf{s}) = 0$. From the definition of the relative scale invariance property, $k(\mathbf{T})$ is a predefined transformation acting as a scaling factor. So, $\Phi(\mathbf{y}) = 0$. 
$\Rightarrow \forall \mathbf{T} \in \mathcal{T}, \Phi(\mathbf{T};\mathbf{s}) = \Phi(\mathbf{I};\mathbf{s}) = 0$ \\
As per this argument, scale invariance is required corresponding to the source components $\mathbf{s}$ and not necessarily with respect to $\mathbf{x}$. This is satisfied by the contrasts measuring $0$ corresponding to independence and satisfying relative scale invariant.
\item By definition, the relation between the measures corresponding to $\mathbf{x}$ components and their scaled version $\mathbf{Tx}$ components is known. 
$\forall \mathbf{T} \in \mathcal{T}, \Phi(\mathbf{T};\mathbf{x}) = k(\mathbf{T})\Phi(\mathbf{I};\mathbf{x})$ \\
This assures the contrast measure for whole equivalence class of solutions are known.
\item For the most BSS algorithms or precisely the orthogonal approach BSS algorithms $\mathbf{y} = \mathbf{Wx}$, where  $\mathbf{W}$ is the estimated unmixing orthogonal rotation transformation and $\mathbf{x}$ are the equivariant uncorrelated (whiten) components. 
This implies that the measure is applied on the solution set that is equally scaled. Mathematically,  $\mathbf{y} = \mathbf{\Lambda x}$, but $\bs{\Lambda}$ is a constant for the whole solution set. Also, corresponding $k(\bs{\Lambda})$ is constant for the whole solution set. 
\end{itemize}
\end{proof} 
Though the relative scale invariance property is sufficient for a quantity to be a contrast, in most of the cases the quantity can be easily converted into a scale invariant quantity. This has been demonstrated for $L^p$ norm of FD, GFD and HFD distance measures. Now, let us verify whether the  derived independence measures qualify to be a contrast or not.
\begin{proposition}
\label{fdcontrast}
 $ \Phi^{\mbox{FD}}_{p} \mbox{ or }  \Phi^{\bs{\Delta}}_p:\mathcal{H}\times \mathcal{H}\cdot \mathcal{S} \rightarrow \mathbb{R}$  is a contrast for linear BSS, where:
\begin{align*}
\Phi^{\mbox{FD}}_{p}(\mathbf{H};\mathbf{x}) \mbox{ or }   \Phi^{\bs{\Delta}}_p(\mathbf{H};\mathbf{x}) = \Phi^{\bs{\Delta}}_p(\mathbf{y})  \stackrel{def}{=} - \bs{\Delta}_p(\mathbf{y}) = - d_p\left(  \prod_{i=1}^{n} p_{y_i}(y_i), p_{\mathbf{y}}(\mathbf{y}) \right) 
\end{align*}
\end{proposition}   
\begin{proof}
Let us verify the scale invariance property of the contrast for both without and with normalization.
Let $\mathbf{T}\in \mathcal{T}$ be $n\times n$ diagonal scaling matrix, as a trivial filter, with the non-zero diagonal entries $t_i, i= 1,\dots,n$.
\begin{align*}
p_{\mathbf{Tx}}(t_1x_1,t_2x_2,\ldots,t_nx_n) &= \frac{1}{|det \mathbf{T}|}p_{\mathbf{x}}(x_1,x_2,\ldots,x_n) \nonumber \\
p_{\mathbf{(Tx)_i}}(\mathbf{(Tx)_i}) &= \frac{1}{|t_i|}p_{x_i}(x_i) \nonumber \\
 \Rightarrow  \prod_{i=1}^{N} p_{x_i}(x_i)  &= \frac{1}{|det \mathbf{T}|}\prod_{i=1}^{N} p_{x_i}(x_i) \nonumber 
\end{align*}
\begin{align*} 
\mbox{Now, } \Phi^{\bs{\Delta}}_p(\mathbf{y}) & = - \bs{\Delta}_p(\mathbf{y}) = -||\bs{\Delta}(\mathbf{y})||_{p}	\nonumber \\
	& = - \left(\int_{\mathbf{y}} \left| \prod_{i=1}^{n}p_{y_i}(y_i) - p_{\mathbf{y}}(\mathbf{y}) \right|^p d\mathbf{y}\right)^{\frac{1}{p}}  \\
	& = - \left(\int_{\mathbf{x}} \left(\frac{1}{|\det\mathbf{T}|}\left| p_{\mathbf{x}}(\mathbf{x}) -  \prod_{i=1}^{n}p_{x_i}(x_i) \right|\right)^p |\det\mathbf{T}| d\mathbf{x}\right)^{\frac{1}{p}}  \\
	& = - \left|det \mathbf{T}\right|^{\frac{1-p}{p}}\bs{\Delta}_p(\mathbf{x})
\end{align*}
This proves that the contrast $\Phi^{\bs{\Delta}}_p(\mathbf{y})$, without normalization of PDFs, is scale invariant for $p=1$ i.e. corresponding to $L^1$-norm of $\bs{\Delta}$. It assures relative scale invariance for $1<p<\infty$.  
As already discussed either the relative scale invariance is a sufficient condition or the measures are applied on normalized  densities (i.e. densities with  zero mean and unit variance) the scale invariance property is satisfied. Corresponding to normalized density, $ t_i = 1, \forall i = 1,2,\ldots,n $. \\
The permutation invariance can be proved in a same way as $|\det \mathbf{T}| = 1$. \\
The Proposition \ref{lpfdmeasure} proves the Dominance property.   \\
By Definition \ref{inddef}, $\bs{\Delta}_{p}(\mathbf{y}) = 0$ if and only if the components  $\mathbf{y}= \mathbf{Hs}$ are independent. So, $\mathbf{H}$ should be a trivial filter in $\mathcal{T}$.
This proves the Discrimination property.
\end{proof}
\hspace{0.2 in}		
Similarly, let us now verify whether the GFD is qualified to be a BSS contrast or not. 
\begin{proposition}
\label{gfdcontrast}
 $ \Phi^{\mbox{GFD}}_{p} \mbox{ or }  \Phi^{\bs{\alpha}}_p:\mathcal{H}\times \mathcal{H}\cdot \mathcal{S} \rightarrow \mathbb{R}$  is a contrast for linear BSS, where:
\begin{align*}
\Phi^{\mbox{GFD}}_{p}(\mathbf{H};\mathbf{x}) \mbox{ or }   \Phi^{\bs{\alpha}}_p(\mathbf{H};\mathbf{x}) = \Phi^{\bs{\alpha}}_p(\mathbf{y})  \stackrel{def}{=} - \bs{\alpha}_p(\mathbf{y}) = - d_p\left( \xi_{\mathbf{y}}(\mathbf{y}),\zeta_{\mathbf{y}}(\mathbf{y}) \right) 
\end{align*}
\end{proposition}   
\begin{proof}
Let us verify the scale invariance property of the contrast for both without and with normalization.
Let $\mathbf{T}\in \mathcal{T}$ be $n\times n$ diagonal scaling matrix, as a trivial filter, with the non-zero diagonal entries $t_i, i= 1,\dots,n$. \\
To simplify, let us start with the gradient of one dimensional transformed variable.
\begin{align*}
Y = aX \Rightarrow p_Y(y) &= \frac{1}{a}p_X\left(\frac{y}{a}\right) \\
 \Rightarrow \frac{dp_Y(y)}{dy} & = \frac{1}{a^2}p_X\left(\frac{y}{a}\right)	\\
 \Rightarrow  \int_{y}\frac{dp_Y(y)}{dy}dy & = \frac{1}{a}\int_{x}\frac{dp_X(x)}{dx}dx
\end{align*}
\begin{align*}
\mbox{Let }\mathbf{y} & = \mathbf{Tx}. \\
\Rightarrow \Phi^{\mbox{GFD}}_{p} & = -\bs{\alpha}_p(\mathbf{y})  = - \left( \sum_{i=1}^{n} \int_{y_i}  \left( \zeta_i(\mathbf{y}) - \xi_i(\mathbf{y}) \right)^p dy_i 
 \right)^{\frac{1}{p}} \\
 & = - \left( \sum_{i=1}^{n} \int_{x_i}  \left| \frac{1}{t_i^2}\left( \zeta_i(\mathbf{x}) - \xi_i(\mathbf{x}) \right) \right|^p t_i dx_i  \right)^{\frac{1}{p}} \\
 & = - \left( \sum_{i=1}^{n} \left| t_i\right|^{1-2p} \left\| \alpha_i(\mathbf{x})\right\|_p \right)^{\frac{1}{p}} \\
\end{align*}
This proves, $\bs{\alpha}_p(\mathbf{y})$, without normalization, is neither scale invariant nor relative scale invariant. 
So, without normalization it can not be a BSS contrast, though being an independence measure.\\ 
But, as already discussed the measures are applied on normalized  densities i.e. densities with  zero mean and unit variance, the scale invariance property is satisfied. Corresponding to normalization, $ t_i = 1, \forall i = 1,2,\ldots,n $. \\
The permutation invariance can be proved in a same way as  $|\det \mathbf{T}| = 1$. \\
The Proposition \ref{alpha1p} proves the Dominance property.   \\
By Property \ref{property1}, $\bs{\alpha}_{p}(\mathbf{y}) = 0$ if and only if the components  $\mathbf{y}= \mathbf{Hs}$ are independent. So, $\mathbf{H}$ should be a trivial filter in $\mathcal{T}$.
This proves the Discrimination property.
\end{proof}
\hspace{0.2 in}		
Similarly, let us decide whether  HFD - with and without normalization is qualified to be a BSS contrast or not. 
\begin{proposition}
\label{hfdcontrast}
 $ \Phi^{\mbox{HFD}}_{p} \mbox{ or }  \Phi^{A}_p:\mathcal{H}\times \mathcal{H}\cdot \mathcal{S} \rightarrow \mathbb{R}$  is a contrast for linear BSS of sources with bounded support, where:
\begin{align*}
\Phi^{\mbox{HFD}}_{p}(\mathbf{H};\mathbf{x}) \mbox{ or }   \Phi^{A}_p(\mathbf{H};\mathbf{x}) = \Phi^{A}_p(\mathbf{y})  \stackrel{def}{=} - A_p(\mathbf{y}) = - d_p\left( \Xi_{\mathbf{y}}(\mathbf{y}),Z_{\mathbf{y}}(\mathbf{y}) \right) 
\end{align*}
\end{proposition}   
\begin{proof}
Let us verify the scale invariance property of the contrast for both without and with normalization.
Let $\mathbf{T}\in \mathcal{T}$ be $n\times n$ diagonal scaling matrix, as a trivial filter, with the non-zero diagonal entries $t_i, i= 1,\dots,n$. \\
To simplify, let us start with the Hessian of one dimensional transformed variable.
\begin{align*}
Y = aX \Rightarrow p_Y(y) &= \frac{1}{a}p_X\left(\frac{y}{a}\right) \\
 \Rightarrow \frac{d^2p_Y(y)}{dy^2} & = \frac{1}{a^3}p_X\left(\frac{y}{a}\right)	\\
 \Rightarrow  \int_{y}\frac{d^2p_Y(y)}{dy^2}dy & = \frac{1}{a^2}\int_{x}\frac{dp_X(x)}{dx}dx
\end{align*}
\begin{align*}
\mbox{Let }\mathbf{y} &= \mathbf{Tx}. \\
A_p(\mathbf{y}) & = \left( \sum_{j=1}^{n} \sum_{i=1}^{n} \int_{y_{ij}}  \left( Z_{ij}(\mathbf{y}) - \Xi_{ij}(\mathbf{y}) \right)^p dy_{ij}  \right)^{\frac{1}{p}} \\
 & = \left( \sum_{i=1}^{n}\sum_{i=1}^{n} \left| t_{i}\right|^{1-3p} \left\| A_{ij}(\mathbf{x})\right\|_p \right)^{\frac{1}{p}} 
\end{align*}
This proves, $A_p(\mathbf{y})$, without normalization, is neither scale invariant nor relative scale invariant. 
So, without normalization it is not a BSS contrast, though being an independence measure.\\
But, as already discussed the measures are applied on normalized  densities i.e. densities with  zero mean and unit variance, the scale invariance property is satisfied. Corresponding to normalization, $ t_i = 1, \forall i = 1,2,\ldots,n $\\
The permutation invariance can be proved in a same way as  $|\det \mathbf{T}| = 1$. \\
The Proposition \ref{alpha2p} proves the Dominance property.   \\
By Property \ref{property2}, $\mathbf{A}_{p}(\mathbf{y}) = 0$ if and only if the components  $\mathbf{y}= \mathbf{Hs}$ are independent. So, $\mathbf{H}$ should be a trivial filter in $\mathcal{T}$.
This proves the Discrimination property.
\end{proof}
\subsection{Local Minima Analysis of the Proposed Contrasts} 		
\label{localmin}
The contrasts defined using $L^p$-norm over FD, GFD and HFD have one more advantage that they do not have any local minima. This is a known property of $L^p$-norm, $p>1$, proved as under: 
\begin{align*}
\frac{d}{d \left\| f(x)\right\|} \left\| f(x) \right\|_p &= p  \left\| f(x) \right\|^{p-1}\\
\therefore \frac{d}{d\left\|f(x)\right\|} \left\| f(x)\right\|_p  = 0 &\Rightarrow \left\|f(x)\right\| = 0 
\Rightarrow f(x) = 0, \forall x 
\end{align*}
So, there is no separate proof required to show that the contrasts $\bs{\Delta}_p(\mathbf{y}(\theta))$, $\bs{\alpha}_p(\mathbf{y}(\theta))$ and $A_p(\mathbf{y}(\theta))$ do not have local minima with respect to the corresponding functions. But, still they may have local minima with respect to $\theta$. Also, the estimation method may add local minima. Actually, it could be easily proved that the contrasts may contain local optima, as under.
\begin{align*}
\nabla \bs{\Delta}_p(\mathbf{y}_0) &= 0 \\
\Rightarrow \bs{\Delta}_p(\mathbf{y}_0) &= c \mbox{ (an arbitrary  constant) }  
\end{align*}
Obviously, as only $c = 0 $ imply independence, other values of c correspond to possible local optima. The more detailed analysis follows as under. \\
Let $\mathbf{x} = (x_1, x_2, \ldots , x_n)^T$ be a bounded random vector and $\delta = (\delta_1, \ldots ,\delta_n)^T$ 
be a `small' random vector. Then, the interest here is in the differential of $\Phi_p^{\bs{\Delta}}$ or  $\left\|\bs{\Delta}(\mathbf{x}+\delta)\right\|_p - \left\|\bs{\Delta}(\mathbf{x})\right\|_p$.
\begin{align*}
\left\|\bs{\Delta}(\mathbf{x}+\delta)\right\|_p - \left\|\bs{\Delta}(\mathbf{x})\right\|_p =  \int_\mathbf{x} \left|  \prod_{i=1}^{n} p_{x_i+\delta_i}(\mathbf{x})  -  p_{\mathbf{x+\delta}}(\mathbf{x}) \right|^p d\mathbf{x} - \int_\mathbf{x} \left|  \prod_{i=1}^{n} p_{x_i}(\mathbf{x}) - p_\mathbf{x}(\mathbf{x}) \right|^p d\mathbf{x} 
\end{align*}
Assuming $\mathbf{t}$ as the support  of all the PDFs,  
\begin{align*}
\left\|\bs{\Delta}(\mathbf{x}+\delta)\right\|_p - \left\|\bs{\Delta}(\mathbf{x})\right\|_p &= \int_\mathbf{t} \left|  \prod_{i=1}^{n} p_{x_i+\delta_i}(\mathbf{t})  -  p_{\mathbf{x+\delta}}(\mathbf{t}) \right|^p d\mathbf{t} - \int_\mathbf{t} \left|  \prod_{i=1}^{n} p_{x_i}(\mathbf{t}) - p_\mathbf{x}(\mathbf{t}) \right|^p d\mathbf{t} \\
& = \int_\mathbf{t} \left|  a  -  b \right|^p  -  \left|  c - d \right|^p d\mathbf{t} \indent \mbox{ using symbolic notations}
\end{align*}
where, $a= \prod_{i=1}^{n} p_{x_i+\delta_i}(\mathbf{t})$, $b = p_{\mathbf{x+\delta}}(\mathbf{t})$, $c= \prod_{i=1}^{n} p_{x_i}(\mathbf{t})$ and $d=p_\mathbf{x}(\mathbf{t})$. \\
\textbf{Let's assume $p=1$}:
\begin{align*}
& \left\|\bs{\Delta}(\mathbf{x}+\delta)\right\|_1 - \left\|\bs{\Delta}(\mathbf{x})\right\|_1 = 0 \\
\Rightarrow  & \mbox{ Either } \int_\mathbf{t} \left|  a  -  b \right| d\mathbf{t} = \int_\mathbf{t} \left|  c  -  d \right| d\mathbf{t} \\
&\mbox{ or }  \left|  a  -  b \right| = \left|  c  -  d \right|, \forall \mathbf{t} \\
&\mbox{ or }   a  =  b \mbox{ and }   c = d, \forall \mathbf{t}
\end{align*}
The condition $\int_\mathbf{t} \left|  a  -  b \right| d\mathbf{t} = \int_\mathbf{t} \left|  c  -  d \right| d\mathbf{t} $ do not assure gradient zero for optimal indicating independence condition.\\
As per $\left|  a  -  b \right| = \left|  c  -  d \right|, \forall \mathbf{t}$, 
four different cases can be thought: 
\begin{align*}
\mbox{ Case I: $a>b$, $c>d$ } &\Rightarrow a-b = c-d \Rightarrow a-c = b-d  \Rightarrow \xi_{\mathbf{x}}(\mathbf{x})= \zeta_{\mathbf{x}}(\mathbf{x})\\
\mbox{ Case II: $a>b$, $c<d$ } &\Rightarrow a-b = -c+d \Rightarrow a+c = b+d \Rightarrow \mbox{ spurious optima }  \\
\mbox{ Case III: $a<b$, $c>d$ } &\Rightarrow -a+b = c-d \Rightarrow a+c = b+d \Rightarrow \mbox{ spurious optima }  \\
\mbox{ Case IV: $a<b$, $c<d$ } &\Rightarrow -a+b = -c+d \Rightarrow a-c = b-d  \Rightarrow \xi_{\mathbf{x}}(\mathbf{x})=\zeta_{\mathbf{x}}(\mathbf{x})
\end{align*}
The Case I and Case IV imply independence but not the other cases.\\
The condition $   a  =  b \mbox{ and }   c = d, \forall \mathbf{t}$ also implies independence. \\
Over all, the analysis implies that the contrast $\Phi_{1}^{\bs{\Delta}} $ may have gradient zero indicating spurious maxima. \\
\textbf{Let's assume $p=2$}:
\begin{align*}
& \left\|\bs{\Delta}(\mathbf{x}+\delta)\right\|_2 - \left\|\bs{\Delta}(\mathbf{x})\right\|_2 = 0 \\
\Rightarrow  & \mbox{ Either } \int_\mathbf{t} \left|  a  -  b \right|^2 d\mathbf{t} = \int_\mathbf{t} \left|  c  -  d \right|^2 d\mathbf{t} d\mathbf{t} \\
&\mbox{ or }  \left|  a  -  b \right|^2 = \left|  c  -  d \right|^2, \forall \mathbf{t} \\
&\mbox{ or }   a =b \mbox{ and }   c = d, \forall \mathbf{t}
\end{align*}
The condition $\int_\mathbf{t} \left|  a  -  b \right|^2 d\mathbf{t} = \int_\mathbf{t} \left|  c  -  d \right|^2 d\mathbf{t}$ do not assure gradient zero for optimal indicating independence condition.\\
As per $\left|  a  -  b \right|^2 = \left|  c  -  d \right|^2, \forall \mathbf{t} \Rightarrow $ 
two different cases can be thought: 
\begin{align*}
\mbox{ Case I:} a-b-c+d= 0  & \Rightarrow a-c = b-d  \Rightarrow \xi_{\mathbf{x}}(\mathbf{x})= \zeta_{\mathbf{x}}(\mathbf{x})\\
\mbox{ Case II: }a-b+c-d = 0 & \Rightarrow a+c = b+d \Rightarrow \mbox{ spurious optima }  
\end{align*}
The Case I imply independence but not the Case II.\\
The condition $   a  =  b \mbox{ and }   c = d, \forall \mathbf{t}$ also implies independence. \\
Over all, the analysis implies that the contrast $\Phi_{2}^{\bs{\Delta}} $ may have gradient zero indicating spurious maxima. \\
Same way, for other values of $p$ also, existence of spurious optima can be proved.\\
Also, in a similar way, possible existence of local optima for contrasts $\Phi_{p}^{\bs{\alpha}}$ and $\Phi_p^A$ can be proved. 
\subsection{FD and its Stochastic Gradient}
 The previous relation of FD, GFD and HFD reminds us the relationship between mutual information and the SFD. As proved by \citet{babaie2004differential}, SFD is the stochastic gradient and can be used to derive differential of mutual information. Also, it has been used to derive that mutual information has no local minima \citep{BabaieZadeh2005975}.  So, it will be desired to investigate whether such results can be obtained with respect to FD, GFD and HFD.

Let us try to obtain differential of FD, in terms of GFD as defined in Section \ref{indm}.
Let $\mathbf{x} = (x_1, x_2, \ldots , x_n)^T$ be a random vector and $\delta = (\delta_1, \ldots ,\delta_n)^T$ 
be a `small' random vector. Then, the interest here is in the differential function of FD that is,  $\bs{\Delta}(\mathbf{x}+\delta) - \bs{\Delta}(\mathbf{x})$.
\begin{align*}
\bs{\Delta}(\mathbf{x}+\delta) - \bs{\Delta}(\mathbf{x}) =   \left(  \prod_{i=1}^{n} p_{x_i+\delta_i}(x_i+\delta_i)  -  p_{\mathbf{x+\delta}}(\mathbf{x}+\delta) \right) -  \left(  \prod_{i=1}^{n} p_{x_i}(x_i) - p_\mathbf{x}(\mathbf{x}) \right)
\end{align*}
Assuming $\mathbf{t}$ as the support  of all the PDFs,  
\begin{align*}
\bs{\Delta}(\mathbf{x}+\delta) - \bs{\Delta}(\mathbf{x}) = \left(  \prod_{i=1}^{n} p_{x_i+\delta_i}(t_i)  -  p_{\mathbf{x+\delta}}(\mathbf{t}) \right) -  \left(  \prod_{i=1}^{n} p_{x_i}(t_i) - p_\mathbf{x}(\mathbf{t}) \right)
\end{align*}
Using Lemma 1 in \citep{babaie2004differential}, the following holds.
\begin{align} 
p_{\mathbf{x+\delta}}(\mathbf{t}) - p_\mathbf{x}(\mathbf{t}) &= - \sum_{i=1}^{n}\frac{\partial}{\partial t_i}\{ E_{\delta_i}\{\delta_i|\mathbf{x}=\mathbf{t} \} p_{\mathbf{x}}(\mathbf{t})\} + o(\delta)	 \\
\label{diffpxdelta}
& = - E_{\delta}\{ \delta^T \zeta_\mathbf{x}(\mathbf{x}) \} + o(\delta)
\end{align}
Same can be applied to the product of the marginal PDFs, itself being a PDF.
\begin{align} 
\prod_{i=1}^{n} p_{x_i+\delta_i}(\mathbf{t}) - \prod_{i=1}^{n}p_{x_i}(\mathbf{t}) &= - \sum_{i=1}^{n}\frac{\partial}{\partial t_i}\{ E_{\delta_i}\{\delta_i|\mathbf{x}=\mathbf{t} \} \prod_{i=1}^{n}p_{x_i}(x_i)\} + o(\delta) 	\nonumber \\
\label{diffpx}
& = - E_{\delta}\{ \delta^T \bs{\xi}_{\mathbf{x}}(\mathbf{x}) \} + o(\delta)
\end{align}
Combining Equation \eqref{diffpxdelta} and Equation \eqref{diffpx}, the  differential function of FD can be given by,
\begin{align*}
\bs{\Delta}(\mathbf{x}+\delta) - \bs{\Delta}(\mathbf{x}) = - E_{\delta}\{ \delta^T \bs{\alpha}_{\mathbf{x}}(\mathbf{x})\} + o(\delta) 
\end{align*}
This is the differential function and to convert it into a number, let us simply integrate it over $\mathbf{t}$.
\begin{align*}
\Rightarrow \nabla \int_{\mathbf{x}}\bs{\Delta}_{\mathbf{x}}(\mathbf{x}) &= \int_{\mathbf{t}}\left( \bs{\Delta}_{\mathbf{x}+\delta}(\mathbf{t}) - \bs{\Delta}_{\mathbf{x}}(\mathbf{t}) \right) d\mathbf{t} = \int_\mathbf{t}E_{\delta}\{ \delta^T \bs{\alpha}_{\mathbf{x}}(\mathbf{x})\} d\mathbf{t} + o(\delta) \\
\Rightarrow \int_{\mathbf{x}} \left( \bs{\Delta}(\mathbf{x}+\delta) - \bs{\Delta}(\mathbf{x})\right) d\mathbf{x} &= \int_\mathbf{x}E_{\delta}\{ \delta^T \bs{\alpha}_{\mathbf{x}}(\mathbf{x})\} d\mathbf{x} + o(\delta)  = \delta^T \int_\mathbf{x} \bs{\alpha}_{\mathbf{x}}(\mathbf{x}) + o(\delta)\\
\Rightarrow \nabla \int_{\mathbf{x}}\bs{\Delta}_{\mathbf{x}}(\mathbf{x})  &=  \lim_{\delta \to 0}\frac{\bs{\Delta}(\mathbf{x}+\delta) - \bs{\Delta}(\mathbf{x})}{\delta} = \int_{\mathbf{x}}\bs{\alpha}_{\mathbf{x}}(\mathbf{x})d\mathbf{x}\\
\mbox{Similarly, } \Rightarrow E \{  \bs{\Delta}( \mathbf{x}+\delta ) - \bs{\Delta}(\mathbf{x})  \} 
& = E \{ E_{\delta} \{ \delta^T \bs{\alpha}_{\mathbf{x}}(\mathbf{x}) \} \} + o(\delta) 
 = \delta^T E\{  \bs{\alpha}_{\mathbf{x}}(\mathbf{x})\} + o(\delta)\\
\Rightarrow \nabla E\{ \bs{\Delta}(\mathbf{x})\} &=  E\{\bs{\alpha}(\mathbf{x})\}
\end{align*}
The above result proves that the GFD ($\bs{\alpha}$) serves as a stochastic gradient of the integrated Function Difference or expectation FD of a random vector. 
So, it could have been easier prove that $\bs{\Delta}(\mathbf{x}+\delta) - \bs{\Delta}(\mathbf{x}) = 0 \Leftrightarrow \bs{\alpha}(\mathbf{x}) = 0$ and that implies independence.  But, the similar can not be proved for their corresponding $l^p$ measures i.e. $\bs{\Delta}_p(\mathbf{x}+\delta) - \bs{\Delta}_p(\mathbf{x}) = 0 \nLeftrightarrow  \bs{\alpha}_p(\mathbf{x}) = 0$ can not be proved.  
The reason  is the contrast defined use the $L^p$-norm of FD and not just the integration or expectation of FD, as this quantities do not assure nonnegativity. So, the effort to prove that the contrasts are without local minima in the previous Section \ref{localmin}, actually resulted into the proof for possible existence of spurious local optima for them. 

\hspace{0.2 in}		
Overall, the contrast $\Psi^{\bs{\Delta}}_p(\mathbf{y}(\theta))$, $\Psi^{\bs{\alpha}}_p(\mathbf{y}(\theta))$ and $\Psi^{A}_p(\mathbf{y}(\theta))$ do not have any local maxima with respect to itself. 
But, it may still have local maxima as a function of $\theta$ (or some other variable), as $\mathbf{y}$ itself is a function of the search parameter $\theta$. 
The next Section \ref{estimate} focuses on the empirical estimation of these contrasts. 
\section{Preliminary background on Estimation of the Derived Contrasts} 
\label{estimate}
Usually, the independence measures avoid estimation of joint PDF, as higher dimension joint PDF estimation is less accurate or requires more samples than marginal PDF estimation. The article \citep{FastHPham03} notes that the measures based on estimation of joint PDF and marginal PDF both, try to cancel out estimation errors compare to the measures only estimating the marginal entropies. 
The minimization of $L^p$-norm of FD, GFD and HFD are the BSS contrasts belong to this class of contrasts. 
The conventional way is to estimate them following a two stage process. In the first stage, separate estimation of joint PDF and marginal PDFs for $\Phi^{\bs{\Delta}}_{p}$, their gradients for $\phi^{\bs{\alpha}}_{p}$ and their Hessians  for $\Phi^A_p$ is achieved. Then, the second stage estimates their difference or $L^p$-norm. The separate estimation of the PDFs and their  derivatives can be achieved through histogram based technique or kernel based method.
The histogram based PDF estimation method is fast but less accurate compare to the kernel method. 
The estimation theory basics says that two stage estimation process for a required quantity amplifies the error in estimation. So, either separate estimation of joint and marginal PDFs and then their difference or the first joint PDF estimation, then based on it the marginal PDFs estimation  and then the difference - this both way are indirect estimation method. Compare to them, the direct estimation of the required quantity from the data is supposed to be less erroneous. 
Though theoretically any real $p\geq 1$ is allowed, either $p=1$ or $p=2$ are more suitable for computation.
The Kernel theory says that a quantity based on the square of the PDF requires less computations than  that based on PDF; if a Gaussian kernel is used. 

In general, compare to the estimation of PDFs, their derivatives and Hessians have more inaccuracies or require more samples for same precision. So, the article derives only the contrasts based on FD and GFD. 
In the light of these observations, there is proposed direct estimation of the $L^2$ based contrasts using `least squares' approach.  	
There are two different estimation approaches based on the sample locations selected to place the kernel basis. The first approach is to select the joint sample  locations to place the multivariate kernel basis. The corresponding estimator for FD is identified as $\Psi_2^{LSFD}$ and that for GFD is identified as $\Psi_2^{LSGFD}$. 
The methods require $O(b^2)$ computations, where $b$ is the number of basis selected. 
The another approach places kernel basis at  selected paired or un-paired sample locations. It  requires $O(b^3)$ computations with better estimations. It is to be noted that the estimation of the same contrasts without the least square based approach requires $O(N^2)$ or $O(N^3)$ order of computations where N is number of samples.  Also, using Fast Gauss Transform (FGT) and Incomplete Cholskey Factorization the computational complexity can be further reduced.  
Similar methods are already in use for direct estimation of density difference \citep{lsdd13}, density ratio \citep{unifiedDR12,relativedr13} and squared loss mutual information \citep{directdistPDF13, MLsmi13, lsmi214}. 
The information potential due to such an arrangement of basis functions is identified as the Reference Information Potential (RIP). The article extends Information field theory to incorporate the new concepts of Reference Information Potential (RIP) and Cross-RIP (CRIP). The concepts are demonstrated, through above four estimators, to be useful to derive closed form  expressions for information field analysis. 
\subsection{Kernel Basics and Information Potential} 
Given N realizations of an unknown PDF $f(x)$, the kernel density estimate ${\hat{f(x)}}$ is given by 
 \begin{equation}
 \label{eqkde}
 \hat{f(x)} = \frac{1}{N}\sum_{i=1}^{N}\frac{1}{h}{K\left(\frac{x-x_i}{h}\right)}
 \end{equation}
 where, $K(u)$ is the kernel function and h is the bandwidth parameter deciding the spread of the kernel. Usually, $K(u)$ is a symmetric, positive definite and bounded function, i.e. it satisfies the following properties:
 \begin{equation*}
 K(u) \geq 0, \mbox{  }\int_{-\infty}^{\infty}{K(u)du} = 1, \mbox{  }\int_{-\infty}^{\infty}{uK(u)Du} = 0, \mbox{  }\int_{-\infty}^{\infty}{u^2K(u)du} = \mu_2(K) < \infty 
 \end{equation*}

It is known that the convolution (symbol `$*$') of two Gaussian functions is still a Gaussian function($G(\cdot,\cdot)$).  In a single dimension,
\begin{equation*}
G(\mathbf{v}, \sigma_1)*G(\mathbf{u}-\mathbf{v}, \sigma_2) = G(\mathbf{u}, \sqrt{\sigma_1^2+\sigma_2^2})
\end{equation*}
Let us use this property to estimate the expectation of the square of PDF. Let the Gaussian kernel be given as, 
\begin{align*}
G_\sigma(x-m_x) = \int_{-\infty}^{\infty}\frac{1}{\sqrt{2\pi}\sigma}\exp^{-\frac{1}{2}\left(\frac{x-m_x}{\sigma}\right)^2}dx
\end{align*}
Then, 
\begin{align*}
\int\{\hat{f(x)}^2\}dx &= \int_{-\infty}^{\infty}\left( \frac{1}{N}\sum_{i=1}^{N}G_\sigma\left(x-x_i\right) \right)^2dx 	\\
		&= \int_{-\infty}^{\infty}\frac{1}{N^2}\sum_{j=1}^{N}\sum_{i=1}^{N}G_\sigma(x-x_i)G_\sigma(x-x_j)dx	\\
		&= \frac{1}{N^2}\sum_{j=1}^{N}\sum_{i=1}^{N}\int_{-\infty}^{\infty}G_\sigma(x-x_i)G_\sigma(x-x_j)dx	\\
		&= \frac{1}{N^2}\sum_{j=1}^{N}\sum_{i=1}^{N}G_{\sigma\sqrt{2}}(x_i-x_j)
\end{align*}
Thus,  the integration of the square of PDF is achieved in a  computationally efficient way, avoiding the continuous integration. The ITL theory has given significance to this property by 
identifying 
it as a quadratic information potential. The details on IP, related independence measure $QMI_{ED}$ and information forces follow in the Appendix \ref{apip}.
\section{Extention to IP Theory}
\label{ripcrip}
 One of the interpretations describes potential as the amount of work done required to bring a unit charge (for electric field) or unit mass (for gravity field) from infinity  to the point in the force field, where infinity implies a point with zero potential. The particle contains amount of potential energy that  has been applied to work against the force.
It is customary in potential theory to think of a reference potential i.e. assuming that a particle is moved from a  reference point in the field instead  from the infinity. This helps analyzing a potential or gravitational field through a reference framework instead absolute. 
In a gravitational field theory, the potential energy at a hight from a sea level reference or some other local reference;  in  an electric field theory potential difference with respect to the common/neutral  of the system or earth - are the respective examples. Moreover, during field analysis it is a general practice to start with a reference potential and then to express the required related quantities as a function of this reference potential.  For example, in nodal analysis for electrical circuit analysis,  reference potential is assumed at every node. 

\hspace{0.2 in}	Once defined IP, the natural query is whether it is possible to derive the concept of reference potential for  information fields? Further, whether  there can be derived some laws for information field analysis using the reference IP concept?  The first question is answered defining RIP and related quantities in this section. The second question is answered in the next Section \ref{lsfd0}. 
\subsection{Reference Information Potential (RIP)}
\label{rip}
 In kernel analysis, it is customary to initially assume a set of kernel basis placed at some selected sample points and then the required quantities are expressed as a function of the basis.  
The potential due to kernel basis can be identified as a Reference Information Potential (RIP). Analogous with the laws in electric circuit analysis, the least squares like approaches can be thought to bring functional relationship between a required quantity and the reference potential. 

\hspace{0.2 in}	Let $\Psi(x) = \{\psi(x_1),\psi(x_2),\ldots,\psi(x_b)\}$ be the set of kernel functions consecutively placed at selected sample locations\footnote{usually, the basis are placed at sample points but can be placed at some other points in the  field} $x_i, i=1,2,\ldots,b$. %
They act as basis as potential at any point in the field is measured using linear combinations of them.  
. The selected sample points can be seen as the occurances of a random variable $X_\psi$.
Then, the potential of $X_\psi$ ($\hat{V_\alpha}(X_\psi)$) is a Reference Information Potential (RIP). More specifically, quadratic RIP is the integral of the square of the PDF of $X_\psi$ ($p_{x_\psi}(x)$), as under:
\begin{align*}
\mbox{RIP}_2 = V_{R}^2&\stackrel{def}{=} \int_x p_{X_\psi}^2(x) dx \\ 
\hat{V}_{R}^2 & = \int_x \hat{p}^2_{X_\psi}(x)dx \\ 
 & = \int_{x}\left( \frac{1}{b}\sum_{i=1}^{b}\psi_\sigma\left(x-x_i\right) \right)^2dx 	\\
		&= \int_{x}\frac{1}{b^2}\sum_{j=1}^{b}\sum_{i=1}^{b}\psi_\sigma(x-x_i)\psi_\sigma(x-x_j)dx	\\
		&= \frac{1}{b^2}\sum_{j=1}^{b}\sum_{i=1}^{b}\int_{x}\psi_\sigma(x-x_i)\psi_\sigma(x-x_j)dx	\\
&= \hat{V_2}(X_\psi) 
\end{align*}
For a Gaussian kernel, $\psi(x_i) = G(x_i)$, the following holds:
\begin{align*}
\hat{V}_{R} &= \frac{1}{b^2}\sum_{j=1}^{b}\sum_{i=1}^{b}\int_{x}G_\sigma(x-x_i)G_\sigma(x-x_j)dx	\\
		&= \frac{1}{b^2}\sum_{j=1}^{b}\sum_{i=1}^{b}G_{\sigma\sqrt{2}}(x_i-x_j)
\end{align*}
The quadratic RIP definition can be generalized to $\alpha$ RIP, as:
\begin{align*}
RIP_\alpha = \mathbf{V}_R^\alpha \stackrel{def}{=} \int_x (\hat{p_{X_\psi}}(x))^\alpha dx 
\end{align*}
Once defined RIP, two more related concepts can be defined to bring the closed form expression for information field analysis. 
\subsection{Cross Reference Information Potential (CRIP)}
The Cross Information Potential (CIP) is defined as the entropy of a PDF $f(x)$ with respect to an another PDF $g(x)$:
$\mbox{CIP} =  E\{f(x)\} = \int f(x)g(x)dx$. With reference to the newly defined RIP concept, entropy of a PDF $f(x)$ with respect to the reference PDF $\hat{p}_{X_\psi}$ is called CRIP. The CRIP estimates the potential on the selected locations as the basis due to the interactions of locations from the sample space of $f(x)$ (or vice versa).
\begin{align*}
\mbox{CRIP}_2 = \bs{\mathcal{V}}_{R}^2(f)  \stackrel{def}{=} \bs{\mathcal{V}}_2(f,X_\psi) & =  \int f(x)p_{X_\psi} dx \\
	\hat{\mathcal{V}}_2(f,X_\psi)  &= \int_{x}\frac{1}{N}\sum_{i=1}^{N}\psi_\sigma(x-x_f(i))\frac{1}{b}\sum_{j=1}^{b} \psi_\sigma(x-x(j))dx	\\
		&= \frac{1}{Nb}\sum_{j=1}^{b}\sum_{i=1}^{N}\int_{x}\psi_\sigma(x-x_f(i))\psi_\sigma(x-x(j))dx	\\
\end{align*}
For a Gaussian kernel, $\psi(x-x_i) = G(x-x_i)$, then:
\begin{align*}
\hat{\mathcal{V}}_{R} &= \frac{1}{Nb}\sum_{j=1}^{b}\sum_{i=1}^{N}\int_{x}G_\sigma(x-x_f(i))G_\sigma(x-x(j))dx	\\
		&= \frac{1}{Nb}\sum_{j=1}^{b}\sum_{i=1}^{b}G_{\sigma\sqrt{2}}(x_f(i)-x(j))
\end{align*} 
\subsection{Information Interaction Matrix (IIM)}
The analysis may not just require the final scalar outcome, but may depend upon the intermediate information interactions. So, let there be defined an Information Interaction Matrix (IIM) as the matrix due to each interaction. There can be IIM for potential, IIM for reference potential and IIM for information forces etc.  For example, the field with $N$ sample points will have $N^2$ interactions that will be the size of the IIM for potential. Similarly, the IIM for reference potential will be of dimension $b \times b$ and IIM for CRIP will be of dimension $N \times b$. 
This is analogous to the Gram Matrix. Let us symbolize ($V_\alpha(x_i,x_j)$) as the potential on $x_j$ due to interaction with $x_i$ and $\mathbf{V}_\alpha(X)$ as the IIM for the potential of random variable $X$. Also,  $V_\alpha(X)$ is already symbolized as the scalar quantity IP of $X$. Accordingly, $V_R^\alpha$ is the reference potential and $\mathbf{V}_R$ is the Reference potential IIM. In short,  $V(x(i),x(j))= \int \psi(x,x(i))\psi(x,x(j))dx$, $[\mathbf{V}_2(X)]_{ij} =  V(x(i),x(j)) $  and $V_2(X) = \frac{1}{N^2}\sum_{j=1}^{N}\sum_{i=1}^{j} [\mathbf{V}_2(X)]_{ij}$. Similarly,  $\mathcal{V}(x_f(i),x(j)) = \int_{x}\psi_\sigma(x-x_f(i))\psi_\sigma(x-x(j))dx	$, $[\bs{\mathcal{V}}_{R}^2]_{ij} = \mathcal{V}(x_f(i),x(j))$  and $\mathcal{V}_R^2(X) = \frac{1}{Nb}\sum_{j=1}^{b}\sum_{i=1}^{b} [\bs{\mathcal{V}}_{R}^2]_{ij}$.
\section[Closed-Form Expression using RIP through Least Squares]{Closed-Form Expression using Reference Potential through Least Squares} 
\label{lsfd0}
The section targets estimation of  the contrast $\Phi^{\mbox{FD}}_{2}$ using closed form expressions, in terms of the RIP and related concepts. 
The mathematical expression is of type closed form if it requires finite number of constants, variables and operations. 
Instead of the conventional two stage approach,  
here, the estimation is achieved directly in a single stage through `least squares'. 
Similar methods are already in use for direct estimation of density difference \citep{lsdd13}, density ratio \citep{unifiedDR12,relativedr13} and squared loss mutual information \citep{directdistPDF13, MLsmi13, lsmi214}. 
Also, it is worth noting that the LpFD or $||\bs{\Delta}(\mathbf{x})|_2$ independence measure, for $\Phi^{\mbox{FD}}_{2}$ contrast, is same  as  the Euclidean Distance based Quadratic Independence Measure ($QIM_{ED}$) by \citet{ITL2010}. 

\hspace{0.2 in}		
The method of `least squares' aims at estimating the model parameters that minimize the sum of the squares of the errors between the true and the estimated quantity. 
Without loss of generality, let us use this approach to estimate  $\bs{\Delta}(\cdot)$ of a two dimensional random vector and then the results  be generalized to higher dimensions. The $\bs{\Delta}(\cdot)$ of a two-dimensional random vector is:
\begin{equation}
\bs{\Delta}(x,y) :=  p_{{xy}}(x,y) - p_{x}(x)p_{y}(y)
\end{equation}
Let  $g(x,y) = {\hat{\bs{\Delta}}(x,y)}$ be the estimated $\bs{\Delta}(x,y)$ and LSFD  be the Least Squares based  FD estimator. Then:
\begin{align}
\mbox{LSFD} &= \int{ \int{(g(x,y) - \bs{\Delta}(x,y))^2}dxdy }	\nonumber \\
		&= \int_y\int_x g(x,y)^2	dxdy - 2 \int_y\int_x g(x,y)\bs{\Delta}(x,y)dxdy	+ \int_y\int_x \bs{\Delta}(x,y)^2dxdy  \nonumber \\
\label{namev2cip}		
		&= V_2(g(x,y)) - 2 \mathcal{V}(g(x,y),\bs{\Delta}(x,y)) \mbox{   }  (\because \mbox{ the last term has no effect on LSFD})
\end{align} 
where, $V_2(g(x,y))= ||\hat{\bs{\Delta}}||_2$ is the potential of the estimated $\bs{\Delta(x,y)}$ or the $QIP_{ED}$ and \\ 
$\mathcal{V}(g(x,y),\bs{\Delta}(x,y)) = ||\hat{\bs{\Delta}}\bs{\Delta}||$ is the cross information potential between the actual and estimated $\bs{\Delta}(x,y)$. So, both quantities represent the required contrast. But, as proved by \citet{lsdd13}, the linear combination of them,  the LSFD is more bias corrected estimator.
\\   
Let us assume further that $g(x,y)$ is given by a linear combination of the selected basis functions placed at the selected sample points.
\begin{equation}
g(x,y) = \sum_{i=1}^{b}{\theta_i\psi_i(x,y)} = \mathbf{\theta}(x,y)^T\Psi(x,y)
\end{equation}
where, $b$ denotes the number of basis functions; $\theta(x,y) = (\theta_1, \theta_b,...,\theta_b)^{T}$ is the parameter vector and $\Psi(x,y) = (\psi_1,\psi_2,...,\psi_b)^{T}$ is the basis function vector. So, with regularization function $R(\mathbf{\theta}) =  \mathbf{\theta}^T\mathbf{\theta}$ and $\lambda$ as the regularization parameter,
\begin{align}
\label{eqlsfd}
\mbox{LSFD}(\mathbf{\theta}) &=  \mathbf{\theta}^{T}\mathbf{V}_R\mathbf{\theta} - 2\mathbf{h}^T\mathbf{\theta} + \lambda\mathbf{\theta}^T\mathbf{\theta} \\
\mbox{where, }  \mathbf{V}_{R (b \times b)} &= \int{\int{\Psi(x,y)\Psi^T(x,y)dx}dy} \nonumber \\ 
\mathbf{h}_{b \times 1} &= \int{\int{\Psi(x,y)(p_{xy}(x,y) - p_x(x)p_y(y))dx}dy} \nonumber \\
h_l & = \mathcal{V}_R( \psi(x_l,y_l), \bs{\Delta} )	\mbox{ or } 
[\mathbf{h}]_l  = \sum_{i=1}^{N}\left[ \bs{\mathcal{V}}_R( \Psi(x_l,y_l), \bs{\Delta} )' \right]_li \nonumber 
\end{align}
The estimator depends upon the IIM for RIP and the IIM for CRIP of $\bs{\Delta}(x,y)$.  
The optimal value of parameter vector $\theta(x,y)$ can be obtained by minimizing the gradient of LSFD. 
\begin{eqnarray}  
\mathbf{\theta}^* &=& \underset{\mathbf{\theta}}{\mbox{argmin}} \mbox{ LSFD}(\mathbf{\theta}) \nonumber\\
\frac{\partial\mbox{LSFD}}{\partial\mathbf{\theta}} &=& \mathbf{V}_R\mathbf{\theta} + \lambda\mathbf{\theta} - \mathbf{h} \nonumber  \\
\label{eqgradlsfd}
\mathbf{\theta}^* &=& (\mathbf{V}_R + \lambda \mathbf{I}_b )^{-1}\mathbf{h}
\end{eqnarray}
where, $\mathbf{I}_b$ is a b-dimensional identity matrix.
Thus, obtaining IIMs $\mathbf{V}_R$ and  $\bs{\mathcal{V}}_R( \Psi(x_l,y_l), \bs{\Delta} )$ gives the parameter vector ($\theta$), least squares estimator (LSFD),   and $ V_2(\hat{\bs{\Delta}}) = ||\bs{\Delta} ||_2 = \mathbf{\theta}^{T}\mathbf{V}_R\mathbf{\theta}$ all.
This justifies the purpose to define the quantities RIP, CRIP and IIMs. 
\section{$\Phi_2^{FD}$ Estimation through Multiplicative Kernel Model}
\label{lsfd1}
Let us use multiplicative Gaussian kernel function as a basis function placed at the selected sample points. So,
\begin{equation}
g(x,y) = \sum_{i=1}^{b}{\theta_iK(x,x_i)L(y,y_i)} = \mathbf{\theta}^{T}[\mathbf{k}(x)\mbox{ o } \mathbf{l}(y)]
\end{equation}
where, $K(x,x_i)$ and $L(y,y_i)$ are the kernel functions at $x_i$ and $y_i$ consecutively; \\
 $k(x) = (K(x,x_1), K(x,x_2), ..., K(x,x_b))^T$ and $l(y) = (L(y,y_1), L(y,y_2), ..., L(y,y_b))^T$ are the kernel vectors and the operator o denotes \textit{Hadamard product}. This gives
\begin{eqnarray*}
[\mathbf{V}_R(x,y)]_{ij} &=& \int{\int{K(x,x_i)L(y,y_i)K(x,x_j)L(y,y_j)dx}dy}	\\	
\Rightarrow \hat{\mathbf{V}}_{R(b\times b)}(x,y) &=&  
 \mathbf{V}_R(x) \mbox{ o } \mathbf{V}_R(y)  \\
\mbox{or }[\hat{\mathbf{V}_R}(x,y)]_{ij} &=& \left(\frac{1}{\sqrt{\pi}2\sigma}\right)^2 \exp{ \left(-\frac{(x_i-x_j)^2}{4\sigma^2} - \frac{(y_i-y_j)^2}{4\sigma^2}\right)}
\end{eqnarray*}
where, $\mathbf{V}_R(x)$ is a $b \times b$ matrix with entries $[\mathbf{V}_R(x)]_{ij} =  K(x,x_i)*K(x,x_j)$ and $\mathbf{V}_R(y)$ is a $b \times b$ matrix with entries $[\mathbf{V}_R(y)]_{ij} = L(y,y_i)*L(y,y_j)$ and $*$ is the symbol for convolution operation.

\hspace{0.2 in}		
The IIM for RIP $\mathbf{V}_{R(b \times b)}$, for an n-dimensional quantity, is obtained using $b^n$  multiplications.  The computations can be reduced by replacing multiplications through additions of the exponents. This will require  square of the $nb^2$ terms, $\frac{nb^2}{2}$ additions of exponents and then taking exponents of $b^2$ terms. 
Now, the sample estimate of $\mathbf{h}$ ($\hat{\mathbf{h}}$) can be obtained as under:
\begin{align}
\mathbf{h}_{(b \times 1)} &= \int{\int{(\mathbf{k}(x)\mbox{ o }\mathbf{l}(y))(p_{xy}(x,y) - p_x(x)p_y(y))dx}dy} \nonumber \\
{\hat{h_l}} &= h'_l -h''_l \nonumber \\
\mbox{where, } h'_l &= \int \int ( K(x,x_l)L(y,y_l) )\left( \frac{1}{N}\sum_{i=1}^{N}{ (K(x,x_i)L(y,y_i))}\right)dxdy \nonumber \\
& = \frac{1}{N}\sum_{i=1}^{N} \mathcal{V}_R(x_i,x_l) \mathcal{V}_R(y_i,y_l) \nonumber \\ 
&= \frac{1}{4\pi\sigma^2N}\sum_{i=1}^{N}{ \exp \left\{-\frac{(x_i-x_l)^2+(y_i-y_l)^2}{4\sigma^2} \right\} } \nonumber \\
\mbox{and  } h''_l &= \int \int ( K(x,x_l)L(y,y_l) ) \left( \frac{1}{N^2}\sum_{j=1}^{N}\sum_{i=1}^{N}{(K(x,x_i)L(y,y_j))} \right) dxdy\nonumber \\
&=\frac{1}{N^2}\sum_{i=1}^{N}{\left( \mathcal{V}_R(x_i, x_l) \left(\sum_{j=1}^{N} \mathcal{V}_R(y_j,y_l)\right)\right)} \nonumber \\
&= \frac{1}{N^2}\left( \sum_{i=1}^{N} \mathcal{V}_R(x_i,x_l)\right)\left(\sum_{j=1}^{N}\mathcal{V}_R(y_j,y_l)\right)
\end{align}
The estimation of $h'_l$ is obtained by replacing the Hadamard product through addition of the exponents, as for the estimation of $\mathbf{V}_R(x,y)$. Each entry requires $n-1$ additions, then exponent followed by $N$ additions.
Each entry  $h''_l$ is obtained through $nN$ additions and $(n-1)$ multiplication. So, over all vector $\mathbf{h}'$ requires $(N+n-1)b$ additions and $Nb$ number of exponents. The estimation of  vector $\mathbf{h}''$ requires $nNb$ additions and $b(n-1)$ multiplications.

\hspace{0.2 in}		
Once $\mathbf{V}_R(x,y)$ and $\mathbf{h}$ are available, the linear coefficients ($\mathbf{\theta}$) can be obtained solving Equation \eqref{eqgradlsfd}. The time complexity for this is $O(b^2)$. 
 Based on Equation \ref{namev2cip}, the required $\Phi_2^{\bs{\Delta}} = - V_2(\hat{\bs{\Delta}})$, the $CIP(p_{x_1x_2}(x_1x_2), p_{x_1}(x_1)p_{x_2}(x_2)) $ and 
the least square approximation error  (LSFD) are estimated. Also, the method estimates both the Function Difference ($\bs{\Delta}$) and $\Phi^{\mbox{FD}}_{2}$ of a random vector simultaneously. 
The time complexity is usually measured in terms of the number of multiplications. 
With this the total multiplication time complexity is only $O(b^2 + b(N+n-1))$. It can be further reduced by taking exponent of values corresponding to $(x_i-x_j)^2 < (3\sigma)^2 $ or $(y_i-y_j)^2 < (3\sigma)^2 $ as zero. 
Though not the time complexity, the performance directly depends upon the number of samples available; specifically in higher dimensions. 
To effectively increase the available samples for estimation, the  next section uses basis placed at both paired and unpaired samples to estimate $\hat{\bs{\Delta}}$. The estimator is  identified as  LSFD2.
\subsection[Kernel Basis Placed at  Paired and Un-paired Samples]{LSFD2 Estimation through Multiplicative Kernel Basis Placed at  Paired and Un-paired Samples}
\label{lsfd2}
The estimation method places the multiplicative kernels as basis at unpaired samples also. This allows the use of \textit{Kronecker} structure to reduce the computational cost. The approximation $g(x,y)$ is defined as:
\begin{eqnarray*}
g(x,y) &=& \sum_{j=1}^{b}\sum_{i=1}^{b} \theta_{ij} K(x,x_i)L(y,y_j) \nonumber \\
&=& \mbox{vec}(\mathbf{\Theta})^T[(\mathbf{I}_b \otimes \mathbf{k}(x))\mbox{ o }(\mathbf{l}(y) \otimes \mathbf{I}_b)]
\end{eqnarray*}
where, $\mathbf{\Theta}$ is a $b \times b$ parameter matrix, vec($\cdot$) is a vectorization function and $\otimes$ implies the \textit{Kronecker} product. 
Accordingly,
\begin{eqnarray}
\mathbf{V}_R{(i*(b-1)+j,k*(b-1)+l)}  &=& \int \int K(x,x_i)L(y,y_j)K(x,x_k)L(y,y_l)dxdy \nonumber \\
\Rightarrow \mathbf{V}_{R(b^n \times b^n)}(x,y) &=&  
 \mathbf{V}_R(y) \otimes \mathbf{V}_R(x)
\end{eqnarray}
where, $\mathbf{V}_R(x)$ is a $b \times b$ matrix with entries $[\mathbf{V}_R(x)]_{ij} =  K(x,x_i)*K(x,x_j)$ and $\mathbf{V}_R(y)$ is a $b \times b$ matrix with entries $[\mathbf{V}_R(y)]_{ij} = L(y,y_i)*L(y,y_j)$.

\hspace{0.2 in}		
The sample estimate of $\mathbf{h}$ ($\hat{\mathbf{h}}$) can be obtained as under:
\begin{eqnarray}
\mathbf{h}_{(b^n \times 1)} &=&  \int{\int{[(\mathbf{I}_b \otimes \mathbf{k}(x))\mbox{ o }(\mathbf{l}(y) \otimes \mathbf{I}_b)] (p_{xy}(x,y) - p_x(x)p_y(y))dx}dy} \nonumber \\
\hat{h}_{(l*(b-1)+l')} &=& \frac{1}{N}\sum_{i=1}^{N}{\mathcal{V}(x_i,x_l)\mathcal{V}(y_i,y_l')} - \frac{1}{N^2}\sum_{j=1}^{N}\sum_{i=1}^{N}{\mathcal{V}(x_i,x_l)\mathcal{V}(y_j,y_l')} \nonumber \\
\mbox{Accordingly, } \hat{\mathbf{h}} &=& \mathbf{h}' - \mathbf{h}''  
\end{eqnarray}
where,
\begin{displaymath}
\mathbf{h}'  = 
 \begin{cases} 
 \frac{1}{N} \mbox{vec}\left( \bs{\mathcal{V}}_R^T(x) \bs{\mathcal{V}}_R(y)\right), & \mbox{ for $n = 2$}\\
 [\mathbf{h}']_{l+ (l'-1)*b+(l''-1)*b^2} = \frac{1}{N^n}\sum_{i=1}^{N}{\mathcal{V}(x_i,x_l)\mathcal{V}(y_i,y_l')\mathcal{V}(z_i,z_l'')}   & \mbox{ for $n=3$} 
 \end{cases}
 \end{displaymath}
 \begin{displaymath}
 \mathbf{h}'' = 
 \begin{cases} 
 [\mathbf{h}'']_{l+(l'-1)*b} 
 = \frac{1}{N^2} \left( \sum_{i=1}^{N}\mathcal{V}(x_i,x_l)\right) \left( \sum_{j=1}^{N}\mathcal{V}(y_j,y_l')\right) & \mbox{ for $n=2$}\\  
 [\mathbf{h}'']_{l+ (l'-1)*b+(l''-1)*b^2} = \frac{1}{N^3} \sum_{i=1}^{N}\mathcal{V}(x_i,x_l)\sum_{j=1}^{N}\mathcal{V}(y_j,y_l')\sum_{k=1}^{N}\mathcal{V}(z_k,z_l'')
 								 &\mbox{ for $n=3$}
\end{cases} 								 
\end{displaymath}		 
The equation of $\mathbf{h}'$ for $n=2$ is not extendibles as it is to $n>2$. The equations for $n=3$ show the way to get generalization for higher dimensions. As explained previously,  the estimation of vector $\mathbf{h}'$ requires $(N+n-1)b^n$ additions and $Nb^n$ number of exponents. The estimation of vector $\mathbf{h}''$ requires $b^2N$ additions and $n(b^n)$ multiplications.\\ 
To estimate $\Phi_{2}^{\bs{\Delta}}$,  the optimal parameter matrix $\mathbf{\Theta}$ is needed. The Equation \eqref{eqgradlsfd} can be written as:
\begin{equation*}
\mathbf{V}_R\mbox{vec}(\mathbf{\Theta}) + \lambda\mbox{vec}(\mathbf{\Theta}) = \mathbf{h} 
\end{equation*}  
This is the famous \textit{discrete Sylvester equation} and requires $O(b^3)$ computations to solve it.\\
Now, the Equation \eqref{eqlsfd} can be given as under:
\begin{eqnarray}
\mbox{LSFD2} &=& \mbox{vec}(\mathbf{\Theta})^T(\mathbf{V}_R(y) \otimes \mathbf{V}_R(x))\mbox{vec}(\mathbf{\Theta}) - 2\mathbf{h}^T\mbox{vec}(\mathbf{\Theta}) \nonumber \\
 &=& \mbox{vec}(\mathbf{\Theta})^T\mbox{vec}(\mathbf{V}_R(x) \mathbf{\Theta} \mathbf{V}_R(y)^T) - 2\mathbf{h}^T\mbox{vec}(\mathbf{\Theta}) \nonumber \\
 &=& \mbox{trace}( \mathbf{\Theta}^T\mathbf{V}_R(x)\mathbf{\Theta}\mathbf{V}_R(y)^T) - 2\mathbf{h}^T\mbox{vec}(\mathbf{\Theta}) \nonumber 
\end{eqnarray}
The computational complexity of above direct estimation of least square error in FD estimation is $O(b^3)$.
Overall, the multiplicative kernel used as basis at sampled and unsampled pairs has $O(b^3 + (N+n)b^n)$ computations but expected to be more accurate specifically, with higher dimensions and less number of samples.
\subsection{A Note on the CRIP Estimations in above Derivations}
The least squares method has been used for direct estimation of density difference and  density ratio, as it is mention in the Section \ref{lsfd0}. One could have noted that the way $\mathbf{h}$ or the $\bs{\mathcal{V}}(x)$ has been calculated in this article is different from it is elsewhere. 
For example, calculation for  $h_l = \int \psi(x,x_i)p_x(x)dx$ in both the ways is demonstrated here.
This article simplifies  $h_l$ as under:
\begin{align*}
h_l &= \int G_\sigma(x,x_j)p_x(x)dx \\
&= \int G_\sigma(x,x_j) \left(\frac{1}{N} \sum_{i=1}^{N}G_\sigma(x,x_i)dx \right)\\
&= \frac{1}{N}\sum_{i=1}^{N}G_{\sqrt{2}\sigma}(x_i,x_j) 
\end{align*}
Intuitively, the kernel interacts with each sample of the PDF $p_x(x)$. The interaction is a convolution resulting into Gaussian with parameter $\sqrt{2}\sigma$. 

The other articles simplify $h_l$ as under:
\begin{align*}
h_l &= \int G_\sigma(x,x_i)p_x(x)dx = \int G_\sigma(x,x_j) \left(\frac{1}{N} \sum_{i=1}^{N}\delta(x,x_i)dx \right)\\
&= \frac{1}{N}\sum_{j=1}^{N} G_\sigma(x_j-x_i)
\end{align*} 
This is also correct, as $h_l$ can be thought as a convolution of the actual PDF with direct delta. 

Overall, both the approaches are correct. But, the first approach is more precise, as better approximates the PDF $p_x(x)$ through Gaussian  kernel, than the delta kernel in the second approach. The empirical results also justify this approach. Actually, the Cross IP estimation in \citep{ITL2010} follows the first approach though it seems conventionally the second approach is more popular. 

\section{$\Phi_2^{GFD}$ Estimation} 
\label{seclsgfd}
The least squares approximation of $\alpha(x,y)= \nabla(\bs{\Delta}(x,y))$ and $A(x,y)= \nabla^2(\bs{\Delta}(x,y))$ can be achieved in the same way as that for $\bs{\Delta}$ in the previous sections. In general, $\bs{\Delta}^{(r)}(x,y)$ is the $r^{th}$ order derivative of $\bs{\Delta}$ and can be estimated using the linear approximation  $g(x,y) := {\hat{\bs{\Delta}}^{(r)}(x,y)} = \theta(x,y)^T \Psi^{(r)}(x,y) $. It is a customary approach to use multiplicative kernels for multivariate density estimation and then multiplicative derivative kernels for the derivative of multivariate density estimation. Accordingly, all major equations for $\bs{\Delta}^{(r)}(x,y)$ estimation remain as they are in $\bs{\Delta}(x,y)$ estimation, with simply the basis $\Psi(x,y)$ replaced by $\Psi^{(r)}(x,y)$. 
For example, with Gaussian kernel $\Psi(x,y) = G_h(x) o G_h(y)$ and $\Psi^{(r)}(x,y) = G_h^{(r)}(x) o G_h^{(r)}(y) = h^{-r}H_r(x)G(x) o h^{-r}H_r(y)G(y)$. 		
Let the least squares estimator for GFD be called LSGFD (Least Square GFD) and be derived as under:
\begin{align}
\mbox{LSGFD} &= \int{ \int{(g(x,y) - \alpha(x,y))^2}dxdy }	\nonumber \\
\label{namev2gcip}		
		&= V_2(g(x,y)) - 2 \mathcal{V}(g(x,y),\bs{\Delta}(x,y)) \mbox{   }  (\because \mbox{ the last term has no effect on LSFD})
\end{align} 
where, $V_2(g(x,y))$ is the $\left\|\hat{\alpha}\right\|_2$ and $\mathcal{V}(g(x,y),\alpha(x,y))$ is $\left\|\hat{\alpha}\alpha\right\|$. So, both the quantities represent the required contrast. But, as proved by \citet{lsdd13} the linear combination of them,  LSGFD, 
is more bias corrected estimator. Also, let
\begin{equation}
g(x,y) = \sum_{i=1}^{b}{\theta_i\psi_i(x,y)} = \mathbf{\theta}(x,y)^T\Psi(x,y)
\end{equation}
where, $b$ denotes the number of basis functions; $\theta(x,y) = (\theta_1, \theta_b,...,\theta_b)^{T}$ is the parameter vector and $\Psi(x,y) = (\psi_1,\psi_2,...,\psi_b)^{T}$ is the basis function vector. So, with regularization function $R(\mathbf{\theta}) =  \mathbf{\theta}^T\mathbf{\theta}$ and $\lambda$ as the regularization parameter,
\begin{align}
\label{eqlsgfd}
\mbox{LSGFD}(\mathbf{\theta}) &=  \mathbf{\theta}^{T}\mathbf{V}_R\mathbf{\theta} - 2\mathbf{h}^T\mathbf{\theta} + \lambda\mathbf{\theta}^T\mathbf{\theta} \\
\mbox{where, }  \mathbf{V}_{R (b \times b)} &= \int{\int{\Psi(x,y)\Psi^T(x,y)dx}dy} \nonumber \\ 
\mathbf{h}_{b \times 1} &= \int{\int{\Psi(x,y)(\nabla p_{xy}(x,y) - \nabla (p_x(x)p_y(y)))dx}dy} \nonumber \\
h_l & = \mathcal{V}_R(\psi(x_l,y_l), \alpha )	\nonumber 
\end{align}
The optimal value of parameter vector $\theta(x,y)$ can be obtained by minimizing the gradient of LSGFD and obtained as 
\begin{align}  
\label{eqgradlsgfd}
\mathbf{\theta}^* = (\mathbf{V}_R + \lambda \mathbf{I}_b )^{-1}\mathbf{h}
\end{align}
where, $\mathbf{I}_b$ is a b-dimensional identity matrix.
\subsection{$\Phi_2^{GFD}$ Estimation through Multiplicative Kernel Model}
\label{lsgfd1}
Let us use multiplicative Gaussian kernel function as a basis function placed at the selected sample points. So,
\begin{equation}
g(x,y) = \sum_{i=1}^{b}{\theta_iK(x,x_i)L(y,y_i)} = \mathbf{\theta}^{T}[\mathbf{k}(x)\mbox{ o } \mathbf{l}(y)]
\end{equation}
The $\mathbf{V}_R(x,y)$ is already derived for LSFD estimator. 
The sample estimate of $\mathbf{h}$ ($\hat{\mathbf{h}}$) can be obtained as under:
\begin{align}
\mathbf{h}_{(b \times 1)} &= \int{\int{(\mathbf{k}(x)\mbox{ o }\mathbf{l}(y))(\nabla p_{xy}(x,y) - \nabla(p_x(x)p_y(y)))dx}dy} \nonumber \\
{\hat{h_l}} &= h'_l -h''_l  \mbox{, where }\nonumber \\
 h'_l &= \int \int ( K(x,x_l)L(y,y_l) )\left( \frac{1}{N\sigma^{2}}\sum_{i=1}^{N}{ \left(\frac{x-x_i}{\sigma}\right)K(x,x_i)\left(\frac{y-y_i}{\sigma}\right)L(y,y_i)}\right)dxdy \nonumber \\
& = \frac{\pi^{d/2}}{(2\sigma)^{d}(N\sigma^d)}\sum_{i=1}^{N} \left(\frac{x_i-x_l}{\sigma}\right)\bs{\mathcal{V}}_R(x_i,x_l) \left(\frac{y_i-y_l}{\sigma}\right)\bs{\mathcal{V}}_R(y_i,y_l) \nonumber \\ 
&= \frac{1}{(2\pi)^{d/2}\left( \sqrt{2}\sigma \right)^{2d}\left( \sqrt{2}\sigma \right)^dN}\sum_{i=1}^{N}{(x_i-x_l)(y_i-y_l) \exp \left\{-\frac{(x_i-x_l)^2+(y_i-y_l)^2}{4\sigma^2} \right\} } \nonumber \\
h''_l &= \int \int ( K(x,x_l)L(y,y_l) ) \left( \frac{1}{N^2\sigma^2}\sum_{j=1}^{N}\sum_{i=1}^{N}{\left(\frac{x-x_i}{\sigma}\right)K(x,x_i)\left(\frac{y-y_j}{\sigma}\right)L(y,y_j)} \right) dxdy\nonumber \\
&= \frac{\pi^{d/2}}{(2\sigma)^{d}(N\sigma)^d}\left( \sum_{i=1}^{N} \left(\frac{x_i-x_l}{\sigma}\right) \mathcal{V}_R(x_i,x_l)\right)\left(\sum_{j=1}^{N} \left(\frac{y_j-y_l}{\sigma}\right) \mathcal{V}_R(y_j,y_l)\right)
\end{align}
Thus, the parameter vector ($\theta$), the scalar value LSGFD $=||\bs{\Delta}||_2 $ and the the $\bs{\Delta}(x,y)$ - all are obtained in terms of the reference and cross reference IIMs. This justifies further the purpose to define the quantities RIP, CRIP and IIMs. 

The interaction matrices for RIP  ($\mathbf{V}_R$) and CRIP ($\VV$) for LSGFD estimation could be  same as those used to estimate the LSFD. But, for more precise estimations it is better to recalculate them using suitable bandwidth parameter for density derivative estimator, which is usually smaller than that used for density estimation.  

\subsection[$\Phi_2^{GFD}$ Estimation through Kernel Basis Placed at  Paired and Un-paired Samples ]{$\Phi_2^{GFD}$ Estimation through Multiplicative Kernel Basis Placed at  Paired and Un-paired Samples}
\label{lsgfd2}
Similar to the LSFD2 estimator, LSGFD2 estimator can be derived using multiplicative kernel basis at unpaired samples and  \textit{Kronecker} structure to achieve precise computation. The approximation $g(x,y)$ is defined as:
\begin{eqnarray*}
g(x,y) &=& \sum_{j=1}^{b}\sum_{i=1}^{b} \theta_{ij} K(x,x_i)L(y,y_j) \nonumber \\
&=& \mbox{vec}(\mathbf{\Theta})^T[(\mathbf{I}_b \otimes \mathbf{k}(x))\mbox{ o }(\mathbf{l}(y) \otimes \mathbf{I}_b)]
\end{eqnarray*}
where, $\mathbf{\Theta}$ is a $b \times b$ parameter matrix, vec($\cdot$) is a vectorization function and $\otimes$ implies the \textit{Kronecker} product. 

The estimation of $\mathbf{V}_R$ is already discussed. The sample estimate of $\mathbf{h}$ ($\hat{\mathbf{h}}$) can be obtained as under:
\begin{eqnarray}
\mathbf{h}_{(b^n \times 1)} &=&  \int{\int{[(\mathbf{I}_b \otimes \mathbf{k}(x))\mbox{ o }(\mathbf{l}(y) \otimes \mathbf{I}_b)] (\nabla p_{xy}(x,y) - \nabla (p_x(x)p_y(y))dx}dy)} \nonumber \\
\hat{h}_{(l*(b-1)+l')} &=& \frac{\pi^{d/2}}{(2\sigma)^{d}(N\sigma^d)}\sum_{i=1}^{N} \left(\frac{x_i-x_l}{\sigma}\right)\bs{\mathcal{V}}_R(x_i,x_l) \left(\frac{y_i-y_l'}{\sigma}\right)\bs{\mathcal{V}}_R(y_i,y_l') \nonumber \\ 
&& - \frac{\pi^{d/2}}{(2\sigma)^{d}(N\sigma)^d}\left( \sum_{i=1}^{N} \left(\frac{x_i-x_l}{\sigma}\right) \mathcal{V}_R(x_i,x_l)\right)\left(\sum_{j=1}^{N} \left(\frac{y_j-y_l'}{\sigma}\right) \mathcal{V}_R(y_j,y_l')\right) \nonumber
\end{eqnarray}
To estimate $\Phi_2^\alpha$,  the optimal parameter matrix $\mathbf{\Theta}$ is needed. The Equation \eqref{eqgradlsfd} can be written as:
\begin{equation*}
\mathbf{V}_R\mbox{vec}(\mathbf{\Theta}) + \lambda\mbox{vec}(\mathbf{\Theta}) = \mathbf{h} 
\end{equation*}  
This is the famous \textit{discrete Sylvester equation} and requires $O(b^3)$ computations to solve it.\\
Now, the Equation \eqref{eqlsfd} can be given as under:
\begin{align}
\mbox{LSGFD2} 
 = \mbox{trace}( \mathbf{\Theta}^T\mathbf{V}_R(x)\mathbf{\Theta}\mathbf{V}_R(y)^T) - 2\mathbf{h}^T\mbox{vec}(\mathbf{\Theta}) \nonumber 
\end{align}
\section{Empirical Verification of the Derived Estimators as Independence Measures }
\label{verifyexp}
The derived four estimators - LSFD, LSFD2, LSGFD and LSGFD2 - need empirical verification. 
A simple test experiment is designed that verifies their ability to separate the independent and dependent signals.  Further testing, as a BSS contrast, has been left for the future sections. 
The estimators need bandwidth parameter $(\mathbf{h})$ selection for multivariate kernel density estimation (KDE) and the regularization parameter $\lambda$. 
Conventionally, the least squares based direct estimation methods use a Cross Validation (CV) method to select both the parameters. 
The CV method is computationally demanding if good number of choices for a free parameter are provided to obtain accuracy in estimation. Instead, the Silverman's \textit{rule-of-thumb} (ROT) \citep{Silverman86},  balancing computation and optimal parameter value, is used for selecting $(\mathbf{h})$. The Experiment uses ROT for  
$\lambda = 0.01$ for the test experiments.\\ 
\textbf{Experiment (Independence Test)}: 
Let there be generated two  uniformly distributed independent signals $X(1,:)$ and $X(2,:)$ with 300 samples each. Let there be generated a dependent signal: $Y(1,:) = \sin (X(1,:)/20*\pi)$. Find the estimated values for the independent signals - $X(1,:)$ and $X(2,:)$ - and dependent signals -  $X(1,:)$ and $Y(1,:)$.   \\
The results are tabulated in the following Table \ref{indtestlsfd}. Each entry in the Table is a mean of 100 trials. The results show that all the estimators are able to give estimator value sufficiently low for independent signals than dependent signals. 
\begin{table}[!ht]
\caption{Performances of the derived independence measures with their estimation techniques: on the test set 
with independence and dependence signals; number of samples 300; kernel bandwidth parameter $h$ using ROT;  regularization parameter $\lambda = 0.01$. The table entries indicate mean of 100 trials. } 
\centering
\footnotesize {
		\begin{tabular}{|c||c|c|c|c|}
		\hline 
		
Test Condition & \bfseries LSFD & \bfseries LSFD2& \bfseries LSGFD & \bfseries LSGFD2   \\
\hline 
\hline
independent signals  &   0.4725e-03  &  0.5057e-03  &  0.2915e-03  &  0.1165e-03 \\
\hline
dependent signals &  0.0180  &  0.0411  &  0.0055  &  0.0091 \\
\hline
\end{tabular}
\label{indtestlsfd} 
}
\end{table} 
\subsection{Empirical Verification of the Derived Estimators as BSS Contrasts}
By definition, statistical independence implies uncorrelatedness (the opposite is true only for Gaussian variable). The  uncorrelated components with zero mean imply orthogonality. So, the ICs with zero mean are also mutually orthogonal. 
A rotation matrix for specific angle, from the infinite set of all n-dimensional rotation	matrices, will be able to transform  a set of zero mean whiten (orthogonal and univariant) components to ICs. So, the BSS problem reduces to estimating a rotation matrix $\mathbf{R}$ giving \textit{m.i.p.} $\mathbf{y}_i(t)$s.
\begin{equation}
\mathbf{Y}^{*}(t) = \underset{\mathbf{W}}{\mbox{argmax}} \mbox{  } \Phi(\mathbf{y}(t)) 
\end{equation}
where, $\Phi(\mathbf{y}(t))$ or $\Phi(\mathbf{Y}(t))$ is the contrast function based on the dependence or independence measure of random vector $\mathbf{y}$. 

It is known that the conventional information theoretic BSS contrasts give local minima for multimodal distributions. So, here an experiment designed to test the local minima of the derived contrast and some of the conventional contrasts for comparision.\\
\textbf{Experiment (Local MInima Analysis for BSS Contrasts)}:    
The experiment is designed to test the existence of spurious local minima in the optimization landscape of the derived contrast for BSS of two \textit{i.i.d.} sources with varying distributions. The number of samples (N) were kept 300. \\
The contrasts tested include the derived LSFD (with ROT for bandwidth parameter selection), LSFD (with ExROT for bandwidth parameter selection), LSFD2 (with ROT for bandwidth parameter selection), LSFD2 (with ExROT for bandwidth parameter selection), LSGFD (with ROT for bandwidth parameter selection), LSGFD (with ExROT for bandwidth parameter selection), LSGFD2 (with ROT for bandwidth parameter selection), LSGFD2 (with ExROT for bandwidth parameter selection) and existing least squares based independence measures LSMI (with Cross-Validation (CV) for bandwidth parameter selection) \citep{MLsmi13}, LSMI2 (with CV for bandwidth parameter selection) \citep{lsmi214} for comparision.
There are defined 21 types of distributions and used first 20 (type a to t) for this experiment. The first 18 types (a to r) of distributions are suggested by \cite{bach2003kernel} and two more skewed types of distributions were added.
The s type is a GGD with skewness $s = -0.25$ (left skewed)  and kurtosis $k = 3.75$ and  the t type is a GGD with skewness $s = 0.75$ (right skewed)  and kurtosis $k = 0$. Both the distributions are generated using Power Method with parameters $b = 0.75031534111078$, $c= -0.02734119591845$, $d = 0.07699282409939$ for s type and $b=1.11251460048528$,  $c=0.17363001955694$ and $d=-.05033444870926$ for t type. 
	The u type is a Gaussian distribution that is added for some other experiment not reported here.
All 21 distributions are shown in the Figure \ref{pdf21}.	
\begin{figure}
\includegraphics[scale=0.37]{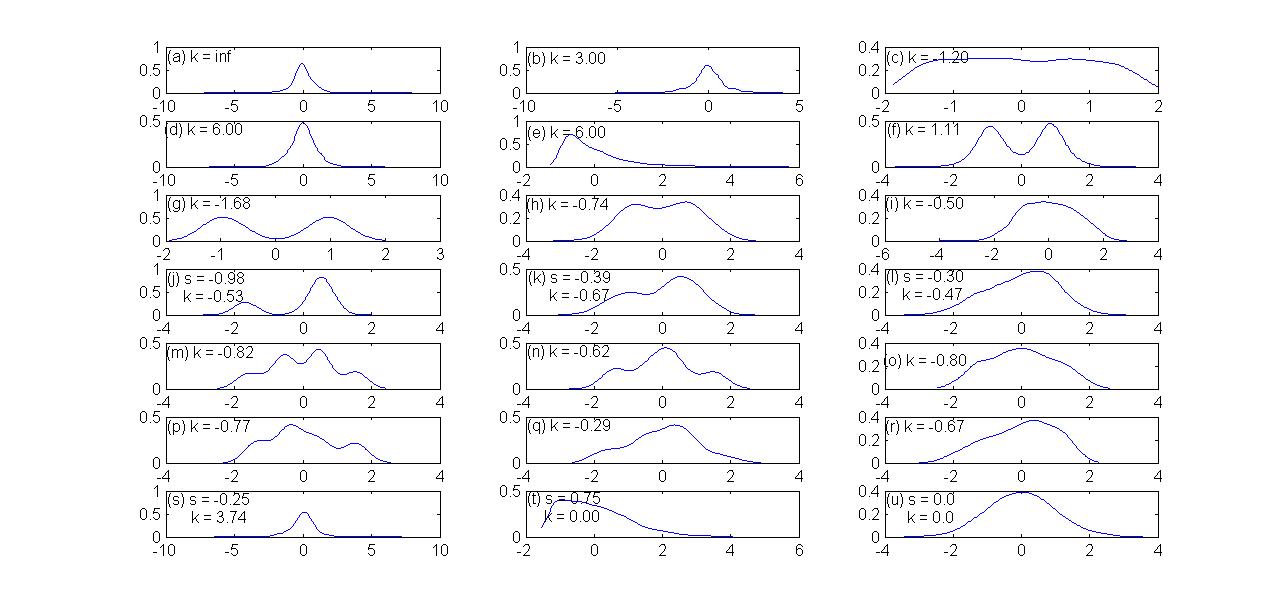} 
\centering
\caption{Probability density functions of sources with their kurtosis: (a) Student with 3 degrees of freedom; (b) double exponential; (c) uniform; (d) Student with 5 degrees of freedom; (e) exponential; (f) mixture of two double exponentials; (g)-(h)-(i) symmetric mixtures of two Gaussians: multimodal, transitional and unimodal; (j)-(k)-(l) nonsymmetric mixtures of two Gaussians, multimodal, transitional and unimodal; (m)-(n)-(o) symmetric mixtures of four Gaussians: multimodal, transitional and unimodal; (p)-(q)-(r) nonsymmetric mixtures of four Gaussians: multimodal, transitional and unimodal; (s) left skewed Generalized Gaussian Distribution(GGD); (t) right skewed GGD; (u) Gaussian distribution}
\label{pdf21}
\end{figure}	 

\subsubsection{Parameter Selection in the Derived Estimators for BSS}
\label{lsfdparameter}
The Experiment justified the use of ROT for bandwidth selection, instead CV for the same.  
But, both the methods have one more problem for BSS like signal processing and machine learning applications. Compare to the applications in previous experiment, where the comparision was at an event or at a point, those applications require to find the most optimal from a given solution set. If CV method is used, there needs to be found new parameter value at every point in consideration. That will be computationally too demanding. 
The ROT assumes Gaussian distribution for the unknown PDF. 
The feasible solution set for the problem is expected to have varying properties, like, varying  distances from Gaussianity and others. Ideally, same bandwidth parameter is not best for all points.  For example; in case of the BSS application, the goal is to find the most non-gaussian (independent) components. For this goal,  assuming Gaussianity for the whole solution set is contradictory and sure way to bring estimation errors. 
This brings the need to use data dependent rules for kernel smoothing parameter that takes into consideration the variation in the distributions of solutions and is also computationally efficient. Such a rule, identified as Extended ROT (ExROT) is derived by \citet{DharmaniExROTarxiv} based on an assumption that the density being estimated is near Gaussian and can be approximated using Gram-Charlier A Series. The rule is used for the contrast estimation in BSS in the following experiment. 
\subsubsection{Results}
The results of the Experiment for local minima analysis are shown in terms of the plots of negative of the contrast value versus the rotation angle theta. The minima of the plots corresponds to the actual sources. Ideally, it should be at $\theta = 0 \mbox{ or } \pi/2$. The plots show the $\theta$ values in radian multiplied by 100. 

The comparative study shows that all the derived estimators show minima at the $\theta = 0 $  and no spurious local minima for unimodal distributions. They have better optimization landscape than the LSMI and LSMI2 contrasts. The LSGFD2 estimator with bandwidth parameter selection has the best performance for multimodal distributions compare to all other contrasts, though it has local minima for distributions Type 4, Type 15 and few others.     
\begin{figure}
\includegraphics[scale=0.35]{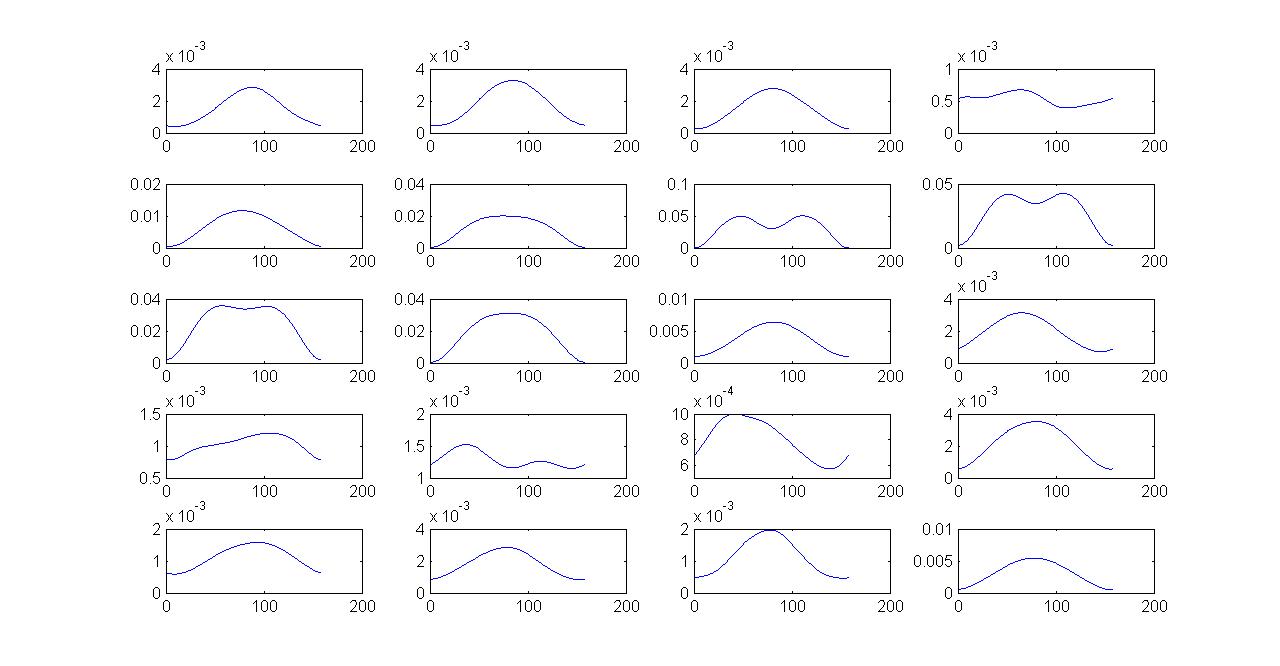} 
\centering
\caption{Plots of LSFD Contrast estimated with bandwidth parameter through ROT versus theta value for the first 20 distributions a-s stacked rowwise}
\label{im202r}
\end{figure}	 
\begin{figure}
\includegraphics[scale=0.35]{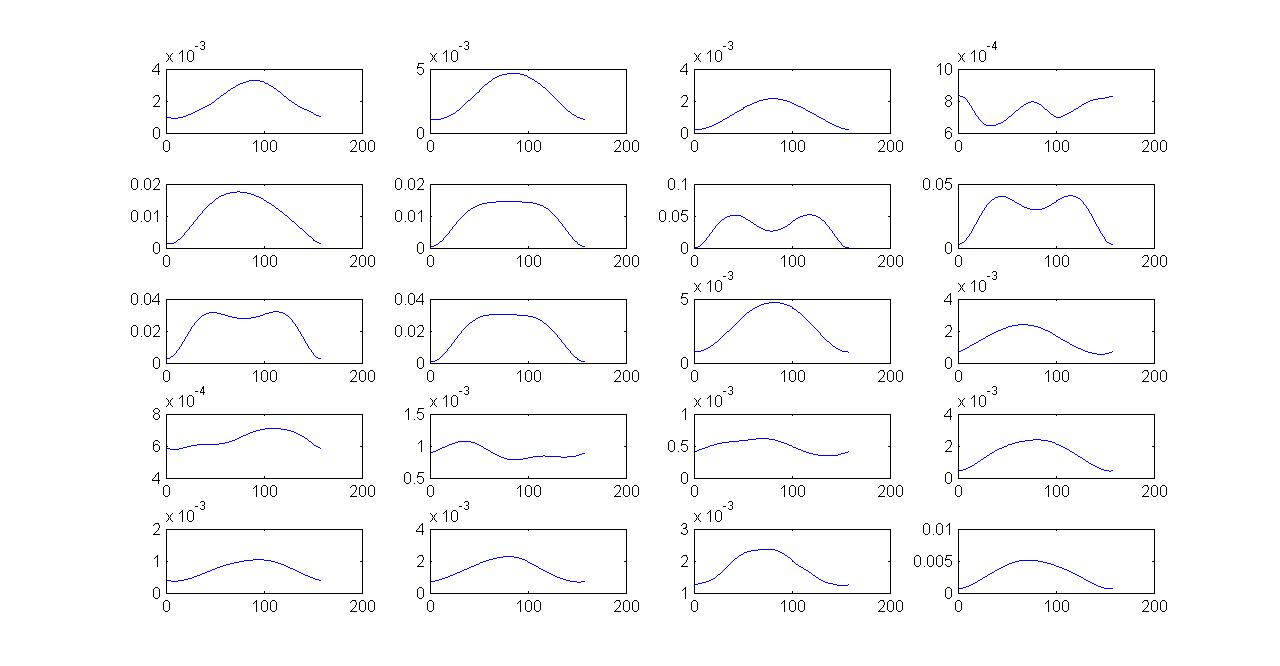} 
\centering
\caption{Plots of LSFD Contrast estimated with bandwidth parameter through ExROT versus theta value for the first 20 distributions a-s stacked rowwise}
\label{im202}
\end{figure}	 
\begin{figure}
\includegraphics[scale=0.35]{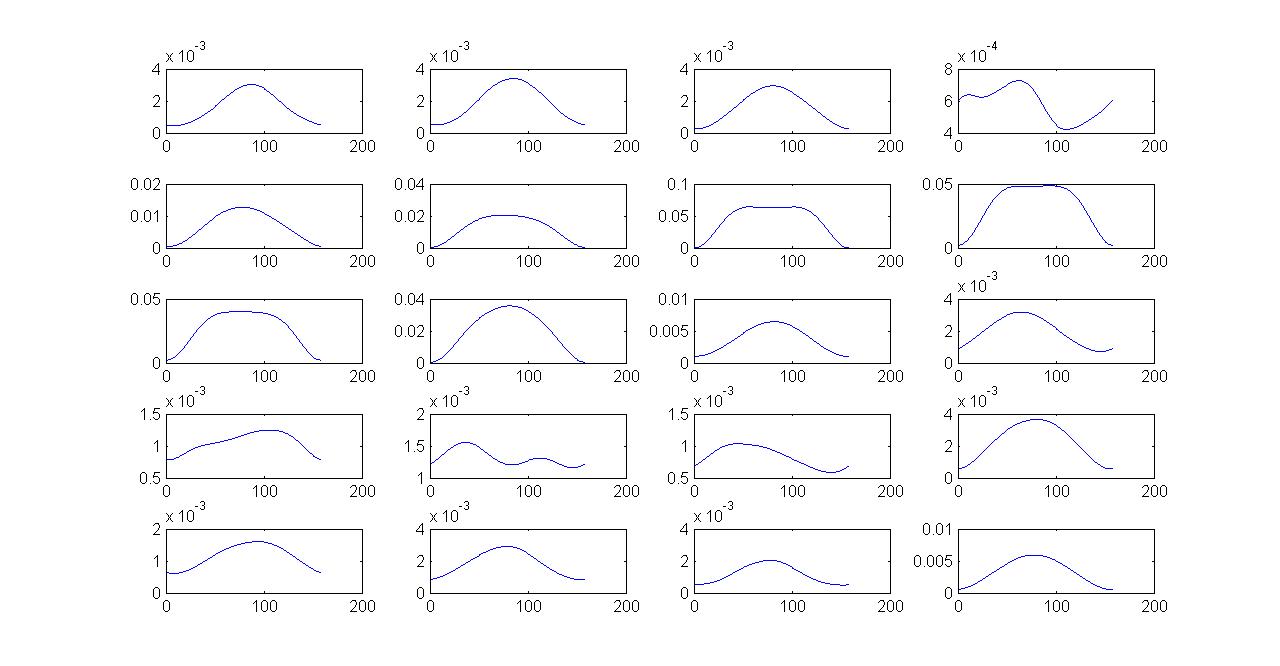} 
\centering
\caption{Plots of LSFD2 Contrast estimated with bandwidth parameter through ROT versus theta value for the first 20 distributions a-s stacked rowwise}
\label{im204r}
\end{figure}	 
\begin{figure}
\includegraphics[scale=0.35]{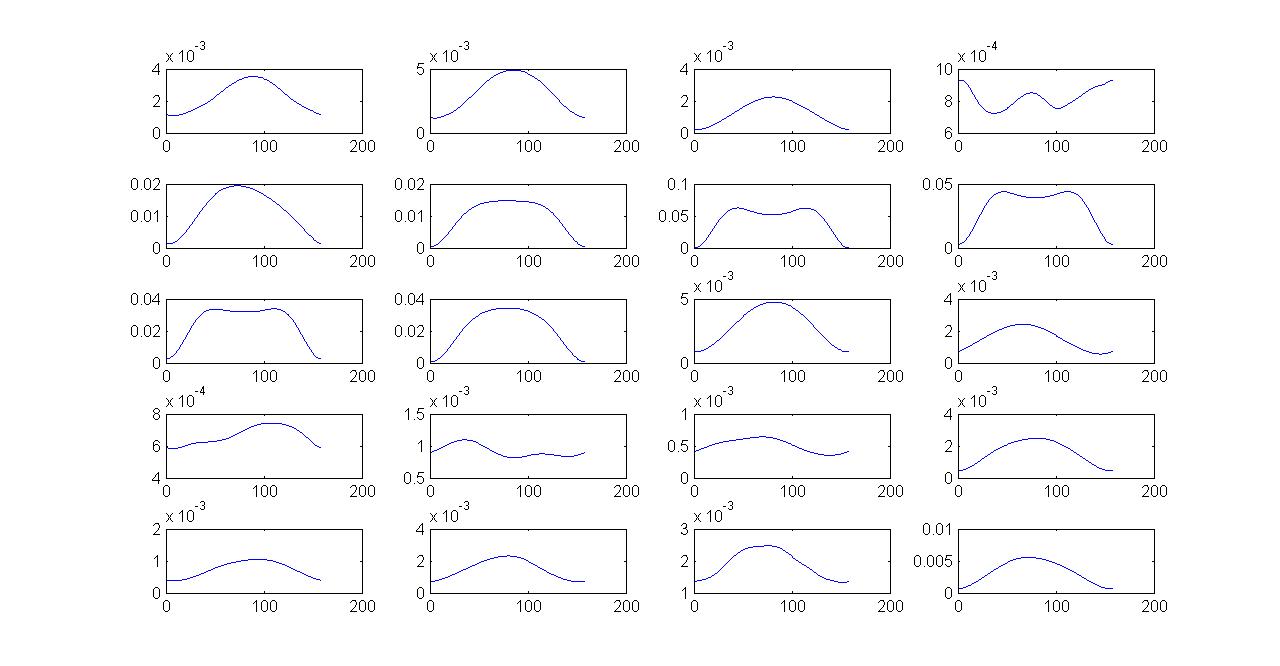} 
\centering
\caption{Plots of LSFD2 Contrast estimated with bandwidth parameter through ExROT versus theta value for the first 20 distributions a-s stacked rowwise}
\label{im204}
\end{figure}	 
\begin{figure}
\includegraphics[scale=0.35]{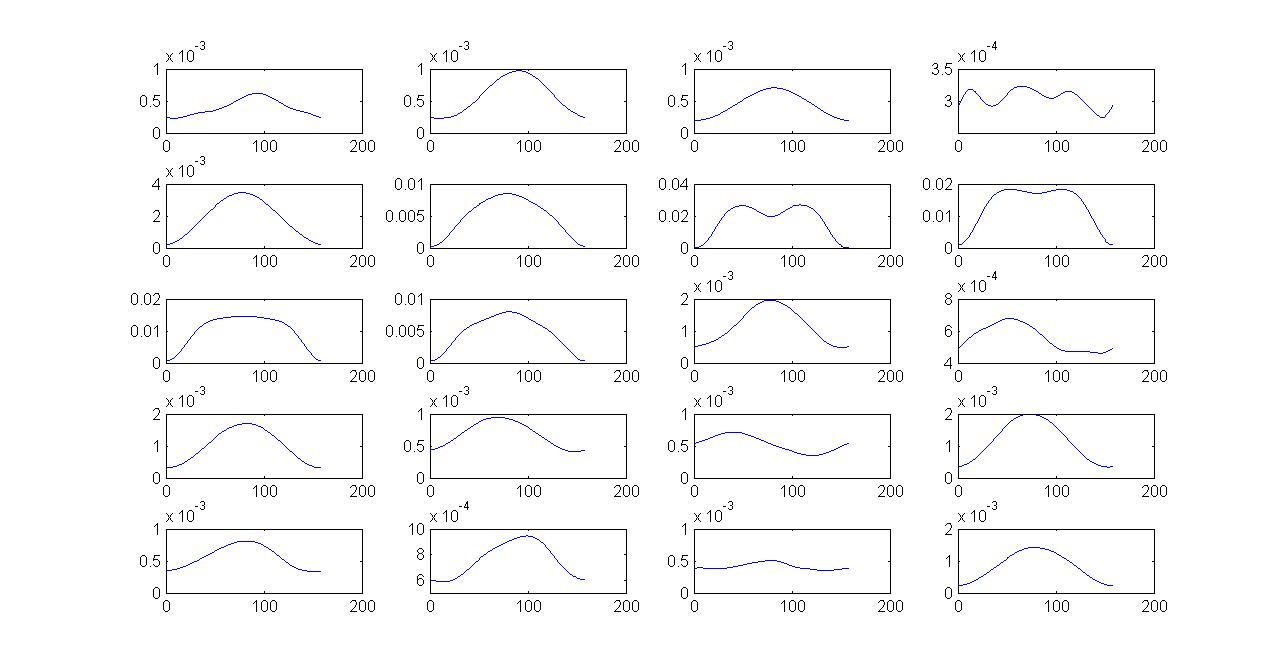} 
\centering
\caption{Plots of LSGFD Contrast estimated with bandwidth parameter through ROT versus theta value for the first 20 distributions a-s stacked rowwise}
\label{im206r}
\end{figure}	 
\begin{figure}
\includegraphics[scale=0.35]{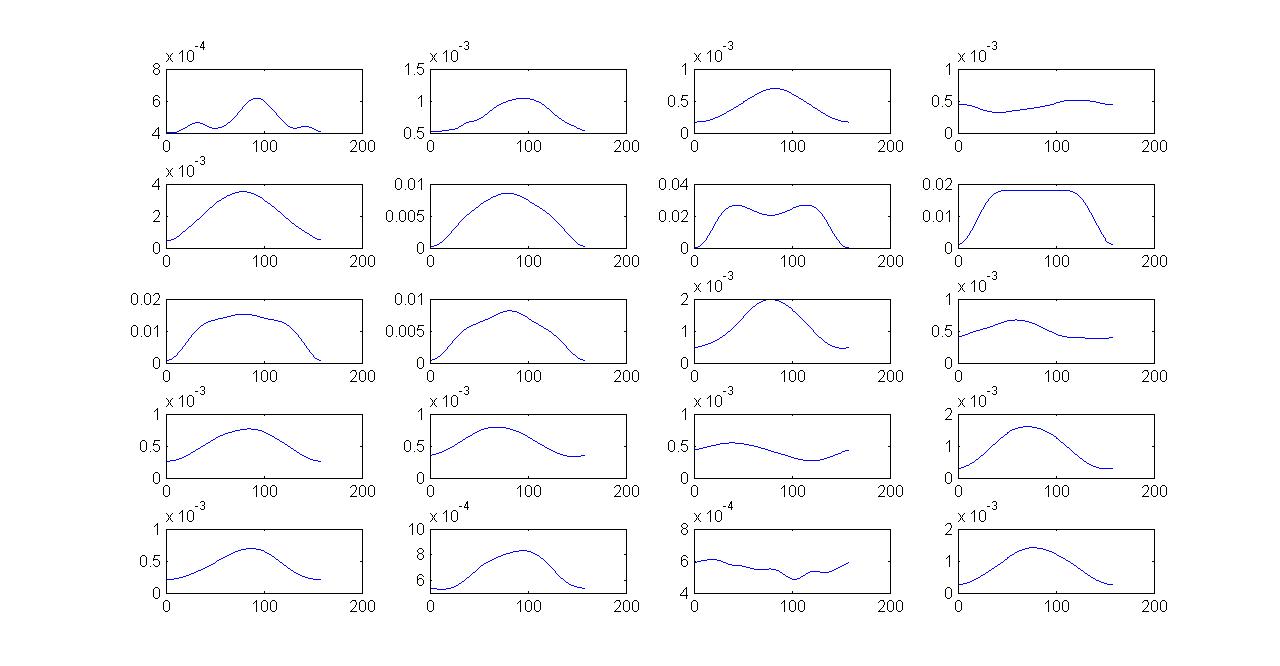} 
\centering
\caption{Plots of LSGFD Contrast estimated with bandwidth parameter through ExROT versus theta value for the first 20 distributions a-s stacked rowwise}
\label{im206}
\end{figure}	 
\begin{figure}
\includegraphics[scale=0.35]{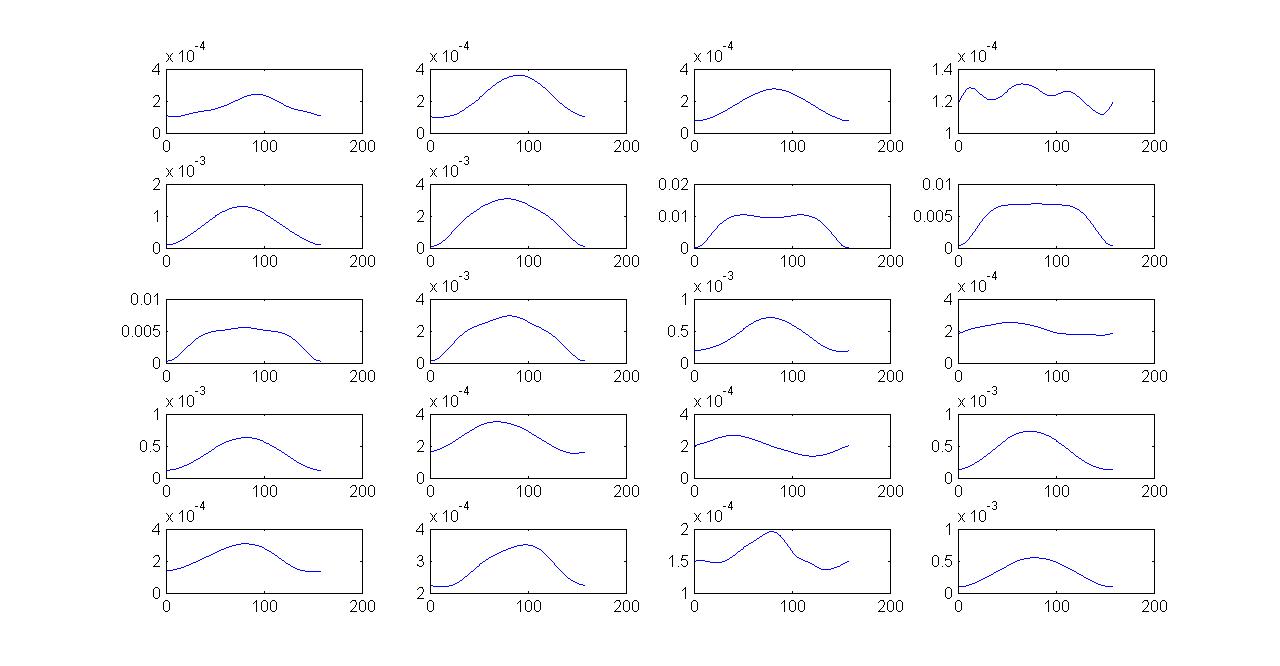} 
\centering
\caption{Plots of LSGFD2 Contrast estimated with bandwidth parameter through ROT versus theta value for the first 20 distributions a-s stacked rowwise}
\label{im208r}
\end{figure}	 
\begin{figure}
\includegraphics[scale=0.35]{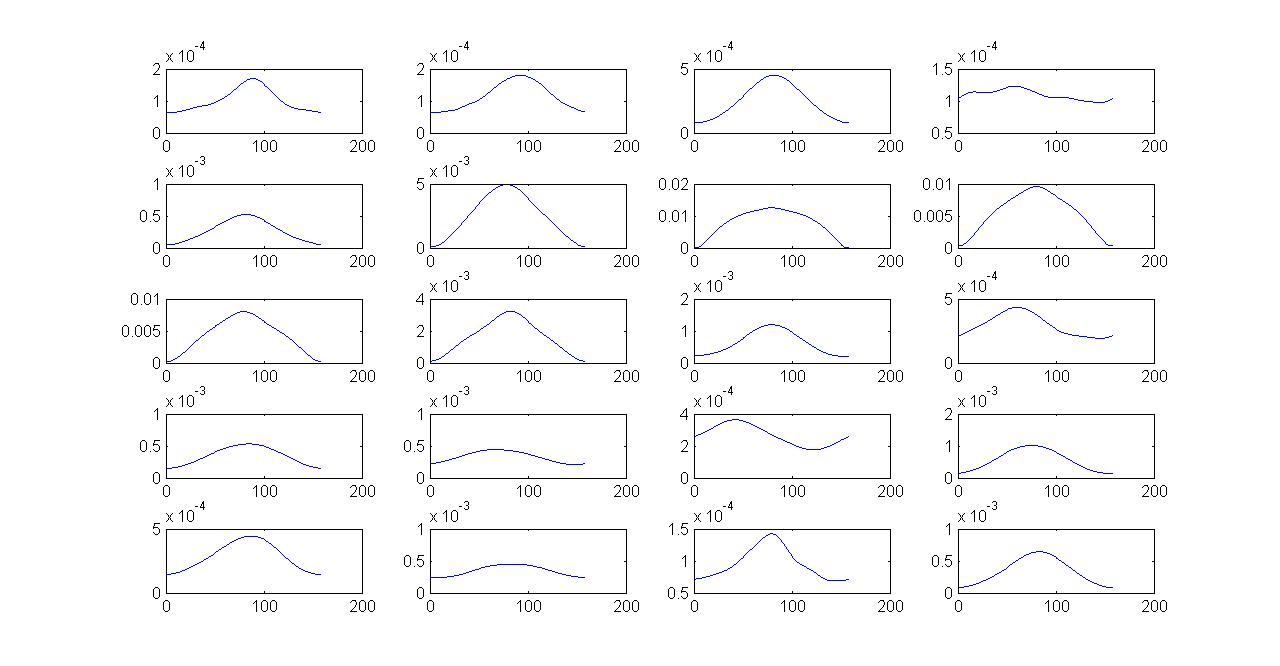} 
\centering
\caption{Plots of LSGFD2 Contrast estimated with bandwidth parameter through ExROT versus theta value for the first 20 distributions a-s stacked rowwise}
\label{im208}
\end{figure}	 
\begin{figure}
\includegraphics[scale=0.35]{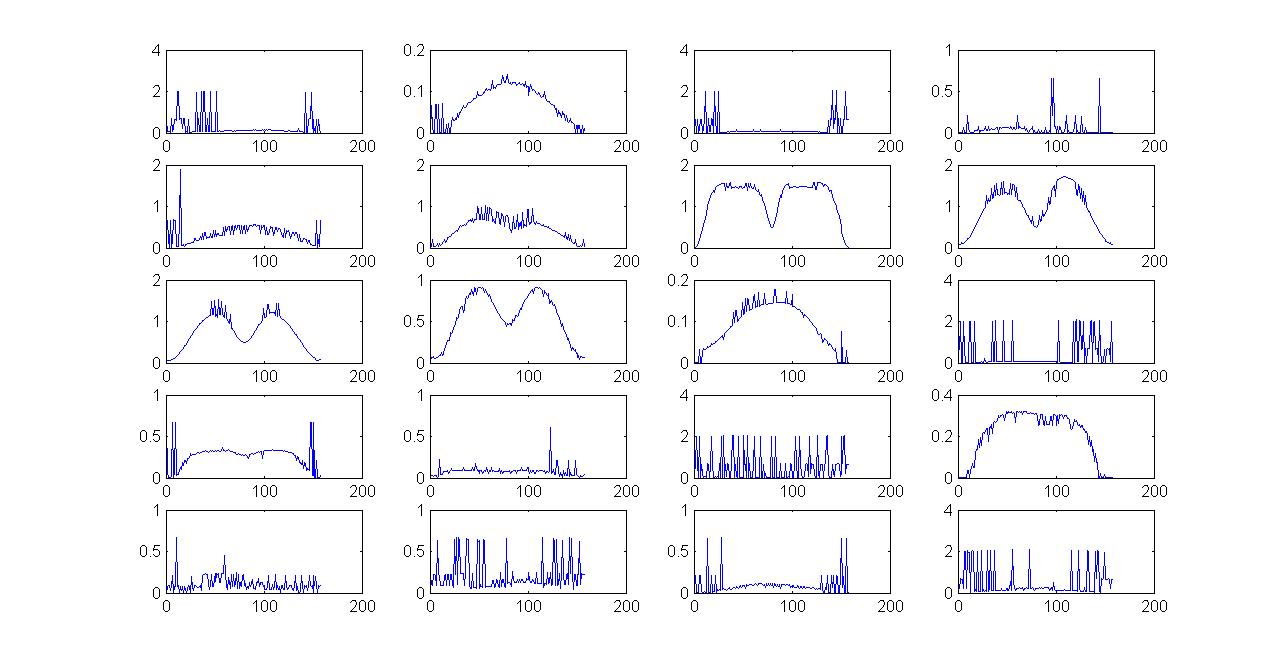} 
\centering
\caption{Plots of LSMI Contrast estimated with bandwidth parameter through CV versus theta value for the first 20 distributions a-s stacked rowwise}
\label{im221}
\end{figure}	 
\begin{figure}
\includegraphics[scale=0.35]{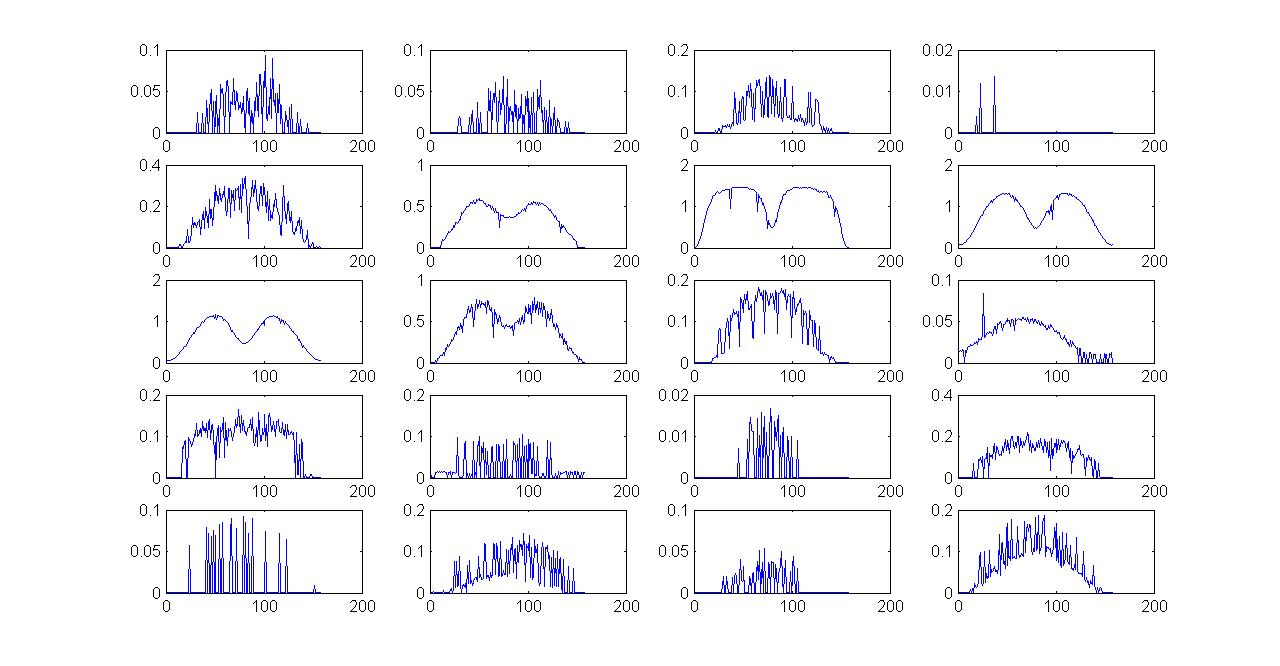} 
\centering
\caption{Plots of LSMI2 Contrast estimated with bandwidth parameter through CV versus theta value for the first 20 distributions a-s stacked rowwise}
\label{im222}
\end{figure}	 
\section{Conclusion}
\label{newIMconclusion}
The article proves that the Gradient Function Difference (GFD) being zero everywhere imply independence. For a bounded support random vector the Hessian Function Difference (HFD) being zero everywhere imply independence. Accordingly, $L^P$ measure of FD, GFD and HFD are proved to be independence measures.
They are used to derive contrast functions for simultaneous ICA and BSS. The contrast functions are proved to satisfy the properties of Scale Invariance, Dominance and Discrimination, avoiding spurious global maxima.   
There has also been derived least squares based two methods to estimate $L^2$ of FD and $L^2$ of GFD contrasts using multiplicative kernel basis. In the first method the basis are placed at only joint samples and in the another method basis are placed at both paired and unpaired samples. The first method requires computations of the order of $O(b^2 + N(b + n - 1))$ and the second method requires that of the order of $O(b^3 + (N+n)b^n))$. But, the second method requires less samples for the same performance. The empirical verification justifies the derived contrasts for BSS applications.  But, further experiments are needed to have the comparision with other contrasts on separation quality against varying number of sources and varying number of samples.  The required performance analysis is restricted here 
and targeted in the future versions of this article. 

\appendix
\section{Measure, Metric, Norm and $L^p$-norm}
\label{aplpnorm}
Measures necessarily assign some nonnegative number to the members of a set in some systematic way. The distance measures or distance functions assign nonnegative value for two elements of a set.
Let there be set $S$. Then, a distance function $d: S \times S  \rightarrow \mathbb{R}$ may satisfy the following conditions for $ x,y,z \in S $:
\begin{enumerate}
\item $d(x,y) \geq 0$ $\mbox{ 	}$ (non-negativity)	
\item $d(x,y) = 0$ iff $x = y$ $\mbox{	}$ (identity of indiscemibles)	
\item $d(x,y) = d(y,x)$	$\mbox{  }$ (symmetry)
\item $d(x,z) \leq d(x,y) + d(y,z) $ $\mbox{  }$ (subadditivity or triangle inequality)
\end{enumerate}
The distance functions satisfying first two conditions are called divergence measures and those satisfying all four conditions are called \textit{metric}. For example, if $S$  contains n-dimensional vectors then $ \forall \mathbf{x},\mathbf{y} \in S$, $p \geq 1$, $d_p: S \times S  \rightarrow \mathbb{R}$ defined as under  is a \textit{metric}. 
\begin{align*}
d_p(\mathbf{x},\mathbf{y}) = \left( \sum_{i=1}^{n} |x_i -y_i|^p \right)^\frac{1}{p}
\end{align*}

\hspace{0.2 in} The concept to derive above \textit{metric} is inspired by the distances in Euclidean geometry. The generalization of this distance measure on sets to that on vector spaces is obtained by defining a norm. Given a vector space $V$ over a field $F$, a norm is a function $\rho: V \rightarrow \mathbb{R}$ with the above four properties of \textit{metric} and added property of absolute Scale Invariance defined as under:
$$ \rho(a\mathbf{x}) = |a|\rho(\mathbf{x}), \forall \mathbf{x}\in V, a \in F$$
For example, given an n-dimensional vector space $\mathbb{R}^n$ and $\mathbf{x} \in \mathbb{R}^n$;  the $L^p$-norm of $\mathbf{x}$ for a real number $p \geq 1$, is defined as:
\begin{align*}
\left\| \mathbf{x}\right\|_{p} = \left(\left|x_1\right|^p + |x_2|^p + \ldots + |x_n|^p	\right)^{\frac{1}{p}}
\end{align*}

\hspace{0.2 in} The same definition has been also extended for functions in $L^p$-spaces. A point in $L^p$-space is an $L^p$ integrable function. 
A function $f: \mathbb{R}^n \rightarrow \mathbb{R}$ is $L^p$ integrable, if $p$-th power of its absolute value is finite, or equivalently,
\begin{align*}
\left\| f(\mathbf{x}) \right\|_{p} = \left( \int_{\mathbb{R}^n}{ |f(\mathbf{x})|^p d\mathbf{x} }\right)^{\frac{1}{p}} < \infty
\end{align*}
It is a complete normed space with all the $L^p$ integrable functions, their linear combinations through real coefficients and including all limit points. 
\section{Information Potential (IP) and related Concepts}
\label{apip} 
In a general sense, potential means an unrealized ability. The gravitational potential and the electric potential are
the known examples from Physics. In both the examples, potential created by a particle (with mass or charge) is inversely proportional to the distance. In kernel density estimation,  a kernel is placed at each sample location and usually kernel is a positive definite function decaying with distance. This fact brings analogy with the potential theory. Each sample is an information particle. 
The PDF is the information potential field in which the information particles interact with each other. In a scalar field, the total potential is the summation of potential due to individual  particles. The information potential (IP) due to the system of samples or the field is given in a same way. 
For a random variable $\mathbf{x}$, the potential on a sample $x_j $  due to other samples, assuming Gaussian kernel,  is given by 
\begin{align*}
V_2(x_j) \stackrel{def}{=} \frac{1}{N}\sum_{i=1}^{N}V_2(x_j,x_i) \mbox{ where, } V_2(x_j,x_i) = G_{\sigma \sqrt{2}}(x_j-x_i)
\end{align*}
So, the IP of $\mathbf{x}$ is 
\begin{align*}
V_2(\mathbf{x}) \stackrel{def}{=} \frac{1}{N}\sum_{j=1}^{N}V_2(x_j) = \frac{1}{N^2}\sum_{j=1}^{N}\sum_{i=1}^{N}V_2(x_j,x_i) = \int\{\hat{f(x)}^2\}
\end{align*}
The quantity $V_2(x)$ or IP  is same as the integration of the square of the PDF. Instead of usual sum in potential theory, the normalization is done to get integral over PDF to be 1. The subscript of $V$ reminds us that this is the quadratic information potential (QIP) as square of the PDF is integrated. The definition is generalized for any $\alpha$ by defining $V_\alpha$ as the integral of $\alpha$ power of the density. Also, instead of a Gaussian kernel any other kernel can be selected. But, they may not have as smooth characteristic as for $\alpha=2$ with Gaussian kernel.
Using this result, ITL theory has defined several scalar descriptors of PDF, that just depend upon the available samples with whole PDF structure into consideration. 

\hspace{0.2 in}		The $\Psi_2^{LSFD}$ defined  in the article,  is already defined as $QMI_{ED}$ by \citep{ITL2010}. 
The quantity $QMI_{ED}$, for a random vector $\mathbf{x} = (x_1, x_2)$, in terms of IP  is derived as under: 
\begin{align*}
QMI_{ED}(x_1, x_2) &= D_{ED}(p_{x_1x_2}(x_1,x_2), p_{x_1}(x_1)p_{x_2}(x_2)) \\
	& = \int_{x_2}\int_{x_1}{(p_{x_1x_2}(x_1,x_2) - p_{x_1}(x_1)p_{x_2}(x_2))^2}dx_1dx_2 \\
	& = \int_{x_2}\int_{x_1}{(p_{x_1x_2}(x_1,x_2) )^2}dx_1dx_2 + \int_{x_2}\int_{x_1}{( - p_{x_1}(x_1)p_{x_2}(x_2))^2}dx_1dx_2 \\
	& \mbox{   }- \int_{x_2}\int_{x_1}{2p_{x_1x_2}(x_1,x_2) p_{x_1}(x_1)p_{x_2}(x_2)}dx_1dx_2  \\
	& = V_J + V_M - 2V_C 
\end{align*}
where, $V_J$ is the IP of the joint PDF, $V_M$ is the potential of the product of the marginal PDFs and $V_C$ is the Cross Information Potential (CIP) similar to the concepts of cross entropy or cross correlation. 

The potentials can be estimated through kernel methods. 
\begin{align*}
\hat{V_J} &= \frac{1}{N^2}\sum_{j=1}^{N}\sum_{i=1}^{N}\hat{V_2}(\mathbf{x}(i),\mathbf{x}(j)) \\
		&= \frac{1}{N^2}\sum_{j=1}^{N}\sum_{i=1}^{N}G_{\sigma\sqrt{2}}(\mathbf{x}(i),\mathbf{x}(j)) \\
		&= \frac{1}{N^2}\sum_{j=1}^{N}\sum_{i=1}^{N}G_{\sigma\sqrt{2}}(x_1(i) - x_1(j))G_{\sigma\sqrt{2}}(x_2(i) - x_2(j)) \\
		&= \frac{1}{N^2}\sum_{j=1}^{N}\sum_{i=1}^{N}\hat{V_2}(x_1(i), x_1(j))\hat{V_2}(x_2(i), x_2(j)) 
\end{align*}
\begin{align*}
\hat{V_M} &= \left(\frac{1}{N^2}\sum_{j=1}^{N}\sum_{i=1}^{N}G_{\sigma\sqrt{2}}(x_1(i)-x_1(j))\right)\left(\frac{1}{N^2}\sum_{j=1}^{N}\sum_{i=1}^{N}G_{\sigma\sqrt{2}}(x_2(i)-x_2(j))\right) \\
& = \hat{V_2}(x_1)\hat{V_2}(x_2) 
\end{align*}
\begin{align*}
\hat{V_C} &= 	\int_{x_2}\int_{x_1}{p_{x_1x_2}(x_1,x_2) p_{x_1}(x_1)p_{x_2}(x_2)}dx_1dx_2 \\
& = \int\int\left[ \frac{1}{N}\sum_{k=1}^{N}G_\sigma(x_1-x_1(k))G_\sigma(x_2- x_2(k))\right] \left[\frac{1}{N}\sum_{j=1}^{N}G_\sigma(x_1-x_1(i))\right] \\
&\mbox{ 	}\left[\frac{1}{N}\sum_{j=1}^{N}G_\sigma(x_2-x_2(j))\right]dx_1dx_2 \\
&= \frac{1}{N}\sum_{i=1}^{N}\frac{1}{N}\sum_{j=1}^{N}\frac{1}{N}\sum_{k=1}^{N}\int G_\sigma(x_1-x_1(i))G_\sigma(x_1-x_1(k))dx_1 \\
& \mbox{ 	}\int G_\sigma(x_2-x_2(j))G_\sigma(x_2-x_2(k))dx_2 \\
&= \frac{1}{N}\sum_{k=1}^{N}\left[ \frac{1}{N}\sum_{i=1}^{N}G_{\sigma \sqrt{2}}(x_1(k) - x_1(i))\right] \left[ \frac{1}{N}\sum_{j=1}^{N}G_{\sigma \sqrt{2}}(x_2(k) - x_2(j))\right] \\
&= \frac{1}{N}\sum_{k=1}^{N} \hat{V_2}(x_1(k)) \hat{V_2}(x_2(k))
\end{align*}
\subsection{Information Forces (IF)} 
It is obvious to think of information forces, once defined the IP. Potential and the force are related concepts.  One of the interpretation of potential is the amount of work done required to bring a unit charge or unit mass from infinity to the point in the force field. The particle contains amount of potential energy that  has been applied to work against the force. The force on sample $x_j$ is the derivative of the IP at a sample with respect to the position of sample $x_j$, that is:
\begin{align*}
\hat{F_2(x_j)}&\stackrel{def}{=}\frac{\partial}{\partial x_j}\hat{V_2}(x_j) \\ &=\frac{1}{N}\sum_{i=1}^{N}G'_{\sigma\sqrt{2}}(x_j-x_i) = \frac{1}{N}\sum_{i=1}^{N}\hat{F_2}(x_j-x_i) \\ 
&= \frac{1}{2N\sigma^2}(x_i-x_j)G_{\sigma\sqrt{2}}(x_j-x_i)
\end{align*}

\vskip 0.2in
\bibliography{Allref}
\end{document}